\documentclass[11pt]{article}
\usepackage[utf8]{inputenc}
\usepackage{authblk}
\usepackage{amsgen,amsmath,amstext,amsbsy,amsopn,amssymb,bbm,listings,parskip,float}
\usepackage{graphicx}
\usepackage[algoruled,boxed,lined]{algorithm2e}
\usepackage{tcolorbox}
\usepackage[font={footnotesize}]{caption}
\usepackage[english]{babel}
\usepackage{amsthm}
\usepackage{enumitem}
\usepackage{pifont}
\usepackage{multirow}
\usepackage{makecell}
\usepackage{mathabx}
\usepackage{microtype}

\newtheorem{theorem}{Theorem}[section]
\newtheorem{corollary}{Corollary}[theorem]
\newtheorem{lemma}[theorem]{Lemma}
\newtheorem{claim}[theorem]{Claim}
\newtheorem{remark}{Remark}[section]
\newtheorem{definition}{Definition}[section]
\newtheorem{assumption}{Assumption}[section]

\usepackage[breaklinks=true]{hyperref}
\usepackage{breakcites}
\hypersetup{
    colorlinks=true,
    linkcolor=blue,
    filecolor=magenta,
    urlcolor=cyan,
    citecolor = blue
}

\newenvironment{subroutine}[1][htb]
  {
   \begin{algorithm}[#1]
  }{\end{algorithm}}

\DeclareMathOperator*{\argmin}{arg\,min}

\newcommand{\R}{\mathbb{R}}

\newcommand{\Ps}{\mathbb{P}}
\newcommand{\E}{\mathbb{E}}
\newcommand{\Ehat}{\widehat{\mathbb{E}}}
\newcommand{\1}{\mathbbm{1}}

\newcommand{\D}{\mathcal{D}}

\newcommand{\X}{\mathcal{X}}
\newcommand{\Px}{\mathcal{P}}
\newcommand{\Pxhat}{\widehat{\mathcal{P}}}
\newcommand{\Hs}{\mathcal{H}}
\newcommand{\A}{\mathcal{A}}
\newcommand{\Y}{\mathcal{Y}}

\newcommand{\yhat}{\hat{y}}

\newcommand{\Yhat}{\widehat{Y}}

\newcommand{\qhat}{\hat{q}}
\newcommand{\qtilde}{\tilde{q}}
\newcommand{\Pshat}{\widehat{\mathbb{P}}}
\newcommand{\fpr}{\text{FP}}
\newcommand{\tpr}{\text{TP}}
\newcommand{\fprhat}{\widehat{\text{FP}}}
\newcommand{\tprhat}{\widehat{\text{TP}}}
\newcommand{\fprtilde}{\widetilde{\text{FP}}}
\newcommand{\tprtilde}{\widetilde{\text{TP}}}

\newcommand{\lamb}{\boldsymbol{\lambda}}
\newcommand{\err}{\text{err} \, }
\newcommand{\errhat}{\widehat{\text{err}} \, }
\newcommand{\errtilde}{\widetilde{\text{err}} \, }

\newcommand{\besth}{\text{BEST}_h}

\newcommand{\Qhat}{\widehat{Q}}
\newcommand{\lambhat}{\widehat{\boldsymbol{\lambda}}}
\newcommand{\Qtilde}{\widetilde{Q}}
\newcommand{\lambtilde}{\widetilde{\boldsymbol{\lambda}}}
\newcommand{\thetatilde}{\widetilde{\boldsymbol{\theta}}}

\newcommand{\rhat}{\widehat{\boldsymbol{r}}}
\newcommand{\rtilde}{\widetilde{\boldsymbol{r}}}
\newcommand{\htilde}{\widetilde{h}}

\title{Differentially Private Fair Learning}
\author[2]{Matthew Jagielski}
\author[1]{Michael Kearns}
\author[1]{Jieming Mao}
\author[2]{Alina Oprea}
\author[1]{\\ Aaron Roth}
\author[1]{Saeed Sharifi-Malvajerdi}
\author[2]{Jonathan Ullman}
\affil[1]{University of Pennsylvania}
\affil[2]{Northeastern University}
\date{\today}

\begin{document}
\maketitle

\begin{abstract}
Motivated by settings in which predictive models may be required to be non-discriminatory with respect to certain
attributes (such as race), but even collecting the sensitive attribute may be forbidden or restricted, we initiate the
study of fair learning under the constraint of differential privacy. We design two learning algorithms that simultaneously promise \emph{differential privacy} and \emph{equalized odds}, a ``fairness'' condition that corresponds to equalizing false positive and negative rates across protected groups. Our first algorithm is a private implementation of the 
equalized odds post-processing approach of \cite{hardt}. This algorithm is appealingly simple, but must be able to use protected group membership explicitly at test time, which can be viewed as a form of ``disparate treatment''. Our second algorithm is a differentially private version of the oracle-efficient in-processing approach of \cite{agarwal} that can be used to find the \emph{optimal} fair classifier, given access to a subroutine that can solve the original (not necessarily fair) learning problem. This algorithm is more complex but need not have access to protected group membership at test time. We identify new tradeoffs between fairness, accuracy, and privacy that emerge only when requiring all three properties, and show that these tradeoffs can be milder if group membership may be used at test time. We conclude with a brief experimental evaluation.
\end{abstract}

\section{Introduction}

Large-scale algorithmic decision making, often driven by machine learning on
consumer data, has increasingly run afoul of various social norms, laws and regulations.
A prominent concern is when a learned model exhibits discrimination against
some demographic group, perhaps based on race or gender. Concerns over such
algorithmic discrimination have led to a recent flurry of research on fairness in machine
learning, which includes both new tools and methods for designing fair models, and studies
of the tradeoffs between predictive accuracy and fairness~\cite{fatstar19}.

At the same time, both recent and longstanding laws and regulations often restrict the use
of ``sensitive'' or protected attributes in algorithmic decision-making. U.S.~law prevents
the use of race in the development or deployment of consumer lending or credit scoring models,
and recent provisions in the E.U.~General Data Protection Regulation (GDPR) restrict or prevent
even the collection of racial data for consumers. These two developments --- the demand for
non-discriminatory algorithms and models on the one hand, and the restriction on the collection
or use of protected attributes on the other --- present technical conundrums, since the most
straightforward methods for ensuring fairness generally require knowing or using the
attribute being protected. It seems difficult to guarantee that a trained model is not discriminating
against (say) a racial group if we cannot even identify members of that group in the data.

A recent line of work~\cite{veale2017fairer, weller} made these cogent observations, and proposed an interesting solution employing
the cryptographic tool of \emph{secure multiparty computation} (commonly abbreviated \emph{MPC}). In this model, we imagine
a commercial entity with access to consumer data that excludes race, but this entity would like to build a
predictive model for, say, commercial lending, under the constraint that the model be non-discriminatory
by race with respect to some standard fairness notion (e.g.~equality of false rejection rates).  In order
to do so, the company engages in MPC with a set of regulatory agencies, which are either trusted parties holding consumers' race data~\cite{veale2017fairer}, or hold among them a secret sharing of race data, provided by the consumers themselves~\cite{weller}.
Together the company and the regulators apply standard fair machine learning
techniques in a distributed fashion.
In this way the company never directly accesses the race data, but still
manages to produce a fair model, which is the output of the MPC.
The guarantee provided by this solution is
the standard one of MPC --- namely, the company learns {\em nothing more than whatever is implied by its
own consumer data, and the fair model returned by the protocol\/}.

Our point of departure stems from our assertion that MPC is the wrong guarantee to give if our
motivation is ensuring that data about an individual's race does not ``leak'' to the company via the model.  In particular, MPC
{\em implies nothing about what individual information can already be inferred from the learned model itself\/}.
The guarantee we would prefer is that the company's data and the fair model do not leak anything
about an individual's race beyond what can be inferred from ``population level'' correlations. That is, the fair model should not leak anything beyond inferences that could be carried out {\em even if the individual in question had declined to provide her racial identity\/}.
This is exactly the type of promise made by {\em differential privacy\/}~\cite{dwork}, but not
by MPC.

\medskip
\textbf{The insufficiency of MPC.}
To emphasize the fact that concerns over leakage of protected attributes
under the guarantee of MPC are more than hypothetical, we describe a natural example where this leakage would actually occur.

{\em \noindent Example. An SVM model, trained in the standard way, is represented by the underlying support vectors, which are just data points from the training data. Thus, if race is a feature represented in the training data, an SVM model computed under MPC reveals the race of the individuals represented in the support vectors. This is the case even if race is uncorrelated with all other features and labels, in which case differential privacy would prevent such inferences. We note that there are differentially private implementations of SVMs. 
}


The reader might object that, in this example, the algorithm is trained to use racial data at test time, and so the output of the algorithm is directly affected by race.  But there are also examples in which the same problems with MPC can arise
{\em even when race is not an input to the learned model, and race is again uncorrelated with the company's data.}
We also note that SVMs are just an extreme case of a learned model fitting, and thus potentially revealing, its training data.  For example, points from the training set can also be recovered from trained neural networks \cite{song2017machine}.

\medskip
\textbf{Our approach: differential privacy.}
These examples show that cryptographic approaches to ``locking up'' sensitive information
during a training process are  insufficient as a privacy mechanism --- {\em we
need to  explicitly reason about what can be inferred from the output of a learning algorithm}, not simply
say that we cannot learn more than such inferences.
In this paper we thus instead consider the problem of designing fair learning algorithms that also promise differential privacy
with respect to consumer race, and thus give strong guarantees about what can be inferred from the learned model.

We note that the guarantee of differential privacy is somewhat subtle, and does {\em not}
promise that the company will be unable to infer race. For example, it might be
that a feature that the company already has, such as zip codes, is perfectly correlated with race, and
a computation that is differentially private might reveal this correlation.
In this case, the company will be able to infer racial information about its customers.
However, differential privacy prevents leakage of individual racial data beyond what can be inferred from population-level correlations. 

Like~\cite{veale2017fairer}, our approach can be viewed as a collaboration between
a company holding non-sensitive consumer data and a regulator holding sensitive data. Our algorithms
allow the regulator to build fair models from the combined data set (potentially also under MPC) in a way that ensures the company, or any other
party with access to the model or its decisions, cannot infer the race of any consumer in the data much more accurately than they could do from population-level statistics alone. Thus, we
comply with the spirit of laws and regulations asking that sensitive attributes not be leaked, while still allowing
them to be used to enforce fairness.



\subsection{Our Results}


We study the problem of learning classifiers from data with protected attributes. More specifically, we are given a class of classifiers $\Hs$ and we output a randomized classifier in $\Delta(\Hs)$ (i.e.~a distribution over $\Hs$). The training data consists of  $m$ individual data points of the form $(X,A,Y)$. Here $X \in \X$ is the vector of unprotected attributes, $A \in \A$ is the protected attribute and $Y \in \{0,1\}$ is the binary label. As discussed above, our algorithms achieve three goals simultaneously:
\begin{itemize}
\item \textbf{Differential privacy}: Our learning algorithms satisfy \emph{differential privacy} \cite{dwork} with respect to protected attributes. (They need not be differentially private with respect to the unprotected attributes $X$ --- although sometimes are.)
\item \textbf{Fairness:} Our learning algorithms guarantee approximate notions of statistical fairness across the groups specified by the protected attribute. The particular statistical fairness notion we focus on is \emph{Equalized Odds} \cite{hardt}, which in the binary classification case reduces to asking that false positive rates and false negative rates be approximately equal, conditional on all values of the protected attribute (but our techniques apply to other notions of statistical fairness as well, including statistical parity).
\item \textbf{Accuracy:} Our output classifier has error rate comparable to non-private benchmarks in $\Delta(\Hs)$ consistent with the fairness constraints.
\end{itemize}

We evaluate fairness and error as in-sample quantities. Out-of-sample generalization for both error and fairness follow from standard sample-complexity bounds in learning theory, and so we elide this complication for clarity (but see e.g. the treatment in \cite{gerrymandering} for formal generalization bounds).


We start with a simple extension of the \textit{post-processing} approach of \cite{hardt}. Their algorithm starts with a possibly unfair classifier $\Yhat$ and derives a fair classifier by mixing $\Yhat$ with classifiers which are based on protected attributes. This involves solving a linear program which takes quantities $\qhat_{\yhat a y}$ as input. Here $\qhat_{\yhat a y}$ is the fraction of data points with $\Yhat = \yhat, A=a, Y=y$. To make this approach differentially private with respect to protected attributes, we start with $\Yhat$ which is learned without using protected attributes and we use standard techniques to perturb the $\qhat_{\yhat a y}$'s before feeding them into the linear program, in a way that guarantees differential privacy. We analyze the additional error and fairness violation that results from the perturbation. Detailed results can be found in Section \ref{sec:postproc}.

Although having the virtue of being exceedingly simple, this first approach has two significant drawbacks. First, even without privacy, this post-processing approach does not in general produce classifiers with error that is comparable to that of the best fair classifiers, and our privacy preserving modification inherits this limitation. Second, and often more importantly, this post-processing approach crucially requires that protected attributes can be used at test time, and this isn't feasible (or legal) in certain applications. Even when it is, if racial information is held only by a regulator, although it may be feasible to train a model once using MPC, it probably is not feasible to make test-time decisions repeatedly using MPC. 

We then consider the approach of \cite{agarwal}, which we refer to it as \textit{in-processing} (to distinguish it from post-processing). They give an \emph{oracle-efficient} algorithm, which assumes access to a subroutine that can optimally solve classification problems absent a fairness constraint (in practice, and in our experiments, these ``oracles'' are implemented using simple learning heuristics). Their approach does not have either of the above drawbacks: it does not require that protected features be available at test time, and it is guaranteed to produce the approximately optimal fair classifier. The algorithm is correspondingly more complicated. The main idea of their approach (following the presentation of \cite{gerrymandering}) is to show that the optimal fair classifier can be found as the equilibrium of a zero-sum game between a ``Learner'' who selects classifiers in $\Hs$ and an ``Auditor'' who finds fairness violations.  This equilibrium can be approximated by iterative play of the game, in which the Auditor plays exponentiated gradient descent and the Learner plays best responses (computed via an efficient cost-sensitive classification oracle). To make this approach private, we add Laplace noise to the gradients used by the Auditor and we let the Learner run the exponential mechanism (or some other private learning oracle) to compute approximate best responses. Our technical contribution is to show that the Learner and the Auditor still converge to an approximate equilibrium despite the noise introduced for privacy. Detailed results can be found in Section \ref{sec:inproc}.

One of the most interesting aspects of our results is an inherent tradeoff that arises between privacy, accuracy, and fairness, that doesn't arise when any two of these desiderata are considered alone. This manifests itself as the parameter ``$B$'' in our in-processing result (see Table \ref{tab:results}) which mediates the tradeoff between error, fairness and privacy. This parameter also appears in the (non-private) algorithm of \cite{agarwal}---but there it serves only to mediate a  tradeoff between fairness and running time. At a high level, the reason for this difference is that without the need for privacy, we can increase the number of iterations of the algorithm to decrease the error to any desired level. However, when we also need to protect privacy, there is an additional tradeoff, and increasing the number of iterations also requires increasing the scale of the gradient perturbations, which may not always decrease error.

This tradeoff exhibits an additional interesting feature. Recall that as we discussed above, the in-processing approach works even if we can not use protected attributes at test time. But \emph{if we are allowed to use protected attributes at test time}, we are able to obtain a better tradeoff between these quantities --- essentially eliminating the role of the variable $B$ that would otherwise mediate this tradeoff. We give details of this improvement in section \ref{sec:extension} (for this result, we also need to relax the fairness requirement from \emph{Equalized Odds} to \emph{Equalized False Positive Rates}). The main step in the proof is to show that, for small constant $B$ and $\Hs$ containing certain ``maximally discriminatory'' classifiers which make decisions \emph{solely} on the basis of group membership, we can give a better characterization of the Learner's strategy at the approximate equilibrium of the zero-sum game.

Finally, we provide evidence that using protected attributes at test time is necessary for obtaining this better tradeoff. In Section \ref{sec:separation}, we consider the sensitivity of computing the error of the optimal classifier subject to fairness constraints. We show that this sensitivity can be substantially higher when the classifier cannot use protected attributes at test time, which shows that higher error must be introduced to estimate this error privately.

\begin{table}[t!]
\begin{center}
\def\arraystretch{3.5}
\resizebox{\columnwidth}{!}{
\begin{tabular}{|c|c|c|c|c|c|c|}
\hline
 Algorithm & \makecell{Assumptions \\ on $\Hs$} & \makecell{Fairness \\ Guarantee} & \makecell{ Needs \\ access to $A$ \\ at test time?} & \makecell{Does it \\ guarantee \\ privacy of \\ $X$ as well?} & \makecell{Error} & \makecell{Fairness Violation} \\
\hline
DP-postprocessing & None & \makecell{Equalized \\ Odds} & \textcolor{red}{Yes} & \textcolor{red}{No} & $\widetilde{O}\left( \frac{|\A|}{m \epsilon} \right)$~\footnotemark
 & $\widetilde{O}\left(\frac{1}{\min \qhat_{ay}  m \epsilon} \right)$  \\
\hline
\multirow{3}{8em}{DP-oracle-learner} & \makecell{$d_{\Hs} < \infty$ \\ $d_{\Hs} := VC(\Hs)$} & \makecell{Equalized \\ Odds} & \textcolor{blue}{No} & \textcolor{red}{No} & $\widetilde{O}\left(\frac{B}{\min \qhat_{ay}} \sqrt{\frac{|\A| d_\Hs }{m \epsilon}}\right)$ & $B^{-1} + \widetilde{O}\left(\frac{1}{\min \qhat_{ay}} \sqrt{\frac{|\A|d_\Hs}{m \epsilon}}\right)$  \\
\cline{2-7}
& $|\Hs|< \infty$ & \makecell{Equalized \\ Odds} & \textcolor{blue}{No} & \textcolor{blue}{Yes} & $\widetilde{O}\left(\frac{B}{\min \qhat_{ay}} \sqrt{\frac{|\A| \ln(|\Hs|) }{m \epsilon}}\right)$ & $B^{-1} + \widetilde{O}\left(\frac{1}{\min \qhat_{ay}} \sqrt{\frac{|\A| \ln(|\Hs|)}{m \epsilon}}\right)$ \\
\cline{2-7}
& \makecell{$| \Hs| < \infty$, \\ $\Hs$ has maximally \\ discriminatory \\ classifiers} & \makecell{Equalized \\ False Positive \\ Rate} & \textcolor{red}{Yes} & \textcolor{blue}{Yes} & $\widetilde{O}\left(\frac{|\A|}{\min \qhat_{ay}} \sqrt{\frac{|\A| \ln (|\Hs|)}{m \epsilon}}\right)$ & $ \widetilde{O}\left(\frac{|\A|}{\min \qhat_{ay}} \sqrt{\frac{|\A| \ln(|\Hs|)}{m \epsilon}}\right)$  \\
\hline
\end{tabular}
}
\end{center}
\caption{Summary of Results for Our Differentially Private Fair Learning Algorithms. In this table, $m$ is the training sample size, $\qhat_{ay}$ is the fraction of data with $A=a$ and $Y=y$, $|\A|$ is the number of protected groups, and $\epsilon$ is the privacy parameter. $B$ is explained in text. For all but the marked error bound, the comparison benchmark is the optimal fair classifier. The marked bound is compared to a weaker benchmark: the outcome of the non-private post-processing procedure.}
\label{tab:results}
\end{table}


\subsection{Related Work}
The literature on algorithmic fairness is growing rapidly, and is by now far too extensive to exhaustively cover here. See \cite{survey} for a recent survey. Our work builds directly on that of \cite{hardt}, \cite{agarwal}, and \cite{gerrymandering}.  In particular, \cite{hardt} introduces the ``equalized odds'' definition that we take as our primary fairness goal, and gave a simple post-processing algorithm that we modify to make differentially private. \cite{agarwal} derives an ``oracle efficient'' algorithm which can optimally solve the fair empirical risk minimization problem (for a variety of statistical fairness constraints, including equalized odds) given oracles (implemented with heuristics) for the unconstrained learning problem. \cite{gerrymandering}  generalize this algorithm to be able to handle infinitely many protected groups. We give a differentially private version of \cite{agarwal} as well.

Our paper is directly inspired by \cite{weller}, who study how to train fair machine learning models by encrypting sensitive attributes and applying secure multiparty computation (SMC). We share the goal of \cite{weller}: we want to train fair classifiers without leaking information about an individual's race through their participation in the training. Our starting point is the observation that differential privacy, rather than secure multiparty computation, is the right tool for this.

We use differential privacy \cite{dwork} as our notion of individual privacy, which has become an influential ``solution concept'' for data privacy in the last decade. See \cite{aaron} for a survey. We make use of standard tools from this literature, including the Laplace mechanism \cite{dwork}, the exponential mechanism \cite{McSherryT07} and composition theorems \cite{DworkKMMN06,DworkRV10}. 

\section{Model and Preliminaries}\label{sec:model}
Suppose we are given a data set of $m$ individuals drawn $i.i.d.$ from an unknown distribution $\Px$ where each individual is described by a tuple $(X, A, Y)$. $X \in \X$ forms a vector of \textit{unprotected attributes}, $A \in \A$ is the \textit{protected attribute} where $|\A| < \infty$, and $Y \in \Y$ is a binary label. Without loss of generality, we write $\A = \{ 0, 1, \ldots, |\A|-1 \}$ and let $\Y = \{0,1\}$. Let $\Pxhat$ denote the empirical distribution of the observed data. Our primary goal is to develop an algorithm to learn a (possibly randomized) \textit{fair} classifier $\Yhat$, with an algorithm that guarantees the \textit{privacy} of the sensitive attributes $A$. By privacy, we mean differential privacy, and by fairness, we mean (approximate versions of) the \textit{Equalized Odds} condition of \cite{hardt}. Both of these notions are parameterized: differential privacy has a parameter $\epsilon$, and the approximate fairness constraint is parameterized by $\gamma$.  Our main interest is in characterizing the tradeoff between $\epsilon$, $\gamma$, and classification error.

\subsection{Notations}
\begin{itemize}
\item $\Ps$ and $\E$ refer to the probability and expectation operators taken with respect to the true underlying distribution $\Px$. $\Pshat$ and $\Ehat$ are the corresponding empirical versions.
\medskip
\item We will use notation $\fpr_a (\Yhat)$ and $\tpr_a (\Yhat)$ to refer to the false and true positive rates of $\Yhat$ on the subpopulation $\{ A=a \}$.
$$
\fpr_a (\Yhat) = \Ps \left[ \Yhat = 1 \, \vert \, A=a, Y=0\right] \quad , \quad \tpr_a (\Yhat) = \Ps \left[ \Yhat = 1 \, \vert \, A=a, Y=1\right]
$$
 $\fprhat_a (\Yhat)$ and $\tprhat_a (\Yhat)$ are used to refer to the empirical false and true positive rates. $\Delta \fpr_a (\Yhat) = \vert \fpr_a (\Yhat) - \fpr_0 (\Yhat) \vert $ and $\Delta \tpr_a (\Yhat) = \vert \tpr_a (\Yhat) - \tpr_0 (\Yhat) \vert$ are used to measure $\Yhat$'s false and true positive rate discrepancies across groups. $\Delta \fprhat_a (\Yhat)$ and $\Delta \tprhat_a (\Yhat)$ are the corresponding empirical versions.
\medskip
\item $\qhat_{\yhat a y} = \Pshat \, [ \Yhat = \yhat, A= a, Y= y ]$ is the empirical fraction of the data with $\Yhat = \yhat, A= a$, and $Y= y$. With slight abuse of notation, we will use $\qhat_{ay} = \Pshat \, [A= a, Y= y ] = \qhat_{1 a y} + \qhat_{0 a y}$ to denote the empirical fraction of the data with $A=a$ and $Y=y$. We will see that $\min_{a,y} \qhat_{ay}$ shows up in our analyses and plays a role in the performance of our algorithms.
\medskip
\item $\errhat (\Yhat) = \Pshat \, [\Yhat \neq Y]$ is the training error of the classifier $\Yhat$. Given a randomized classifier $Q \in \Delta (\Hs)$, $\errhat (Q) = \E_{h \sim Q} \left[ \Pshat \, [h(X) \neq Y] \right]$.
\end{itemize}
\subsection{Fairness}
\begin{definition}[$\gamma$-Equalized Odds Fairness]\label{eo}
We say a classifier $\Yhat$ satisfies the $\gamma$-Equalized Odds condition with respect to the attribute $A$, if for all $a,a' \in \A$, the false and true positive rates of $\Yhat$ in the subpopulations $\{ A=a \}$ and $\{ A=a' \}$ are within $\gamma$ of one another. In other words, for all $a,a' \in \A$,
\begin{align*}
\left\vert \, \fpr_a (\Yhat) - \fpr_{a'} (\Yhat)  \right\vert \le \gamma
\quad , \quad
\left\vert \, \tpr_a (\Yhat) -  \tpr_{a'} (\Yhat)  \right\vert \le \gamma
\end{align*}
The above constraint involves quadratically many inequalities in $|\A|$. It will be more convenient to instead work with a slightly different formulation of $\gamma$-Equalized Odds in which we constrain the difference between false and true positive rates in the subpopulation $\{ A=a \}$ and the corresponding rates for $\{ A=0 \}$ to be at most $\gamma$ for all $a \neq 0$. The choice of group $0$ as an anchor is arbitrary and without loss of generality. The result is a set of only linearly many constraints. For all $a \in \A$:
\begin{align*}
\Delta \fpr_a (\Yhat) = \left\vert \, \fpr_a (\Yhat) - \fpr_{0} (\Yhat) \right\vert \le \gamma
\quad , \quad
\Delta \tpr_a (\Yhat) = \left\vert \, \tpr_a (\Yhat) -  \tpr_{0} (\Yhat) \right\vert \le \gamma
\end{align*}
\end{definition}
\medskip
Since the distribution $\Px$ is not known, we will work with empirical versions of the above quantities, in which all the probabilities will be taken with respect to the empirical distribution of the observed data $\Pxhat$. Since we will generally be dealing with this definition of fairness, we will use the shortened term ``$\gamma$-\textit{fair}" throughout the paper to refer to ``$\gamma$-\textit{Equalized Odds fair}".

\subsection{Differential Privacy}
Let $\D$ be a \emph{data universe} from which a database $D$ of size $m$ is drawn and let $M$ be an algorithm that takes the database $D$ as input and outputs $M(D) \in \mathcal{O}$. Informally speaking, differential privacy requires that the addition or removal of a single data entry should have little (distributional) effect on the output of the mechanism. In other words, for every pair of \emph{neighboring} databases $D \sim D' \in \D^m$ that differ in at most one entry, differential privacy requires that the distribution of $M(D)$ and $M(D')$ are ``close" to each other where closeness are measured by the privacy parameters $\epsilon$ and $\delta$.
\medskip
\begin{definition}[$(\epsilon, \delta)$-Differential Privacy (DP) \cite{dwork}]\label{dp}
A randomized algorithm $M: \mathcal{D}^m \to \mathcal{O}$ is said to be $(\epsilon, \delta)$-differentially private if for all pairs of neighboring databases $D, D' \in \mathcal{D}^m$ and all $O \subseteq \mathcal{O}$,
$$
\Ps \left[ M(D) \in O\right] \le e^{\epsilon} \, \Ps \left[ M(D') \in O \right] + \delta
$$
where $\Ps$ is taken with respect to the randomness of $M$. if $\delta = 0$, $M$ is said to be $\epsilon$-DP.
\end{definition}
\medskip

Recall that our data universe is $\D = (\X, \A, \Y)$, which will be convenient to partition as $(\X, \Y) \times \A$. Given a dataset $D$ of size $m$, we will write it as a pair $D = (D_I, D_S)$ where $D_I \in (\X,\Y)^m$ represents the insensitive attributes and $D_S \in \A^m$ represents the sensitive attributes. We will sometimes incidentally guarantee differential privacy over the entire data universe $\D$ (see Table \ref{tab:results}), but our main goal will be to promise differential privacy only with respect to the sensitive attributes. Write $D_S \sim D'_S$ to denote that $D_S$ and $D'_S$ differ in exactly one coordinate (i.e. in one person's group membership). An algorithm is $(\epsilon,\delta)$-\emph{differentially private in the sensitive attributes} if for all $D_I \in (\X,\Y)^m$ and for all $D_S \sim D_S' \in \A^m$ and for all $O \subseteq \mathcal{O}$, we have:
$$
\Ps \left[ M(D_I, D_S) \in O\right] \le e^{\epsilon} \, \Ps \left[ M(D_I, D'_S) \in O \right] + \delta
$$

Differentially private mechanisms usually work by deliberately injecting perturbations into quantities computed from the sensitive data set, and used as part of the computation. The injected perturbation is sometimes ``explicitly" in the form of a (zero-mean) noise sampled from a known distribution, say Laplace or Gaussian, where the scale of noise is calibrated to the sensitivity of the query function to the input data. However, in some other cases, the noise is ``implicitly" injected by maintaining a distribution over a set of possible outcomes for the algorithm and outputting a sample from that distribution. The \textit{Laplace} or \textit{Gaussian} mechanisms which are two standard techniques to achieve differential privacy follow the former approach by adding Laplace or Gaussian noise of appropriate scale to the outcome of computation, respectively. The \textit{Exponential} mechanism instead falls into the latter case and is often used when an object, say a classifier, with optimal utility is to be chosen privately. In the setting of this paper, to guarantee the privacy of the sensitive attribute $A$ in our algorithms, we will be using the Laplace and the Exponential Mechanisms which are briefly reviewed below. See \cite{aaron} for a more detailed discussion and analysis.

Let's start with the Laplace mechanism which, as stated before, perturbs the given query function $f$ with zero-mean Laplace noise calibrated to the $\ell_1$-sensitivity of the query function. The $\ell_1$-sensitivity of a function is essentially how much a function would change in $\ell_1$ norm if one changed at most one entry of the database.
\medskip
\begin{definition}[$\ell_1$-sensitivity of a function]\label{sensitivity}
The $\ell_1$-sensitivity of $f: \mathcal{D}^m \to \R^k$ is
$$
\Delta f = \max_{\overset{D,D' \, \in \, \mathcal{D}^m}{D \sim D'}} \left\| f(D) - f(D') \right\|_1
$$
\end{definition}
\medskip
\begin{definition}[Laplace Mechanism \cite{dwork}]\label{laplace}
Given a query function $f: \D^m \to \R^k$, a database $D \in \D^m$, and a privacy parameter $\epsilon$, the Laplace mechanism outputs:
$$
\widetilde{f}_{\epsilon} \left(D \right) = f \left(D \right) + \left(W_1, \ldots, W_k \right)
$$
where $W_i$'s are $i.i.d.$ random variables drawn from $\text{Lap} \left( \Delta f/ \epsilon \right)$.
\end{definition}
Keep in mind that besides having privacy, we would like the privately computed query $\widetilde{f}_{\epsilon} (D)$ to have some reasonable accuracy. The following theorem which uses standard tail bounds for a Laplace random variable formalizes the tradeoff between privacy and accuracy for the Laplace mechanism.
\medskip
\begin{theorem}[Privacy vs. Accuracy of the Laplace Mechanism \cite{dwork}]\label{laplacethm}
The Laplace mechanism guarantees $\epsilon$-differential privacy and that with probability at least $1-\delta$,
$$
|| \widetilde{f}_\epsilon \left(D \right) - f \left(D \right) ||_\infty \le \ln \left(\frac{k}{\delta} \right) \cdot \left( \frac{\Delta f}{\epsilon} \right)
$$
\end{theorem}
While the Laplace mechanism is often used when the task at hand is to calculate a bounded numeric query (e.g. mean, median), the Exponential mechanism is used when the goal is to output an object (e.g. a classifier) with maximum utility (i.e. minimum loss). To formalize the exponential mechanism, let $\ell : \D^m \times \Hs \to \R$ be a loss function that given an input database $D \in \D^m$ and $h \in \Hs$, specifies the loss of $h$ on $D$ by $\ell \left(D,h \right)$. Without a privacy constraint, the goal would be to output $\argmin_{h \in \Hs} \ell \left(D,h \right)$ for the given database $D$, but when privacy is required, the private algorithm must output $\argmin_{h \in \Hs} \ell \left(D,h \right)$ with some ``perturbation" which is formalized in the following definition. Let $\Delta \ell$ be the sensitivity of the loss function $\ell$ with respect to the database argument $D$. In other words,
$$
\Delta \ell = \max_{h \,\in \, \Hs} \max_{\overset{D, D' \, \in \, \mathcal{D}^m}{D \sim D'}} \left\vert \ell \left(D,h \right) - \ell \left(D',h \right) \right\vert
$$

\begin{definition}[Exponential Mechansim \cite{McSherryT07}]
Given a database $D \in \D^m$ and a privacy parameter $\epsilon$, output $h \in \Hs$ with probability proportional to $\exp \left(-\epsilon \ell(D,h)/ 2 \Delta \ell  \right)$.
\end{definition}
\medskip
\begin{theorem}[Privacy vs. Accuracy of the Exponential Mechanism \cite{McSherryT07}]\label{expthm}
Let $h^\star = \argmin_{h \in \Hs} \ell \left(D,h \right)$ and $\htilde_\epsilon \in \Hs$ be the output of the Exponential mechanism. We have that $\htilde_\epsilon$ is $\epsilon$-DP and that with probability at least $1-\delta$,
$$
| \ell \, (D,\htilde_\epsilon ) - \ell \left(D, h^\star \right) | \le \ln \left(\frac{|\Hs|}{\delta} \right) \cdot \left( \frac{2 \Delta \ell}{\epsilon} \right)
$$
\end{theorem}
We will discuss some important properties of differential privacy such as \textit{post-processing} and \textit{Composition Theorems} in Appendix \ref{privacyapp}.

\section{Differentially Private Fair Learning: Post-processing}\label{sec:postproc}


In this section we will present our first differentially private fair learning algorithm which will be called \textbf{DP-postprocessing}. The \textbf{DP-postprocessing} algorithm is a private variant of the fair learning algorithm introduced in \cite{hardt} where decisions made by an arbitrary base classifier $\Yhat$ have their false and true positive rates equalized across different groups $\{A=a\}$ in a post-processing step.  Due to the desire for privacy of the sensitive attribute $A$, we assume the base classifier $\Yhat$ is trained only on the unprotected attributes $X$ and that $A$ is used only for the post-processing step. 

The proposed algorithm of \cite{hardt} derives a fair classifier $\Yhat_p$ by mixing $\Yhat$ with classifiers depending on the protected attributes. $\Yhat_p$ is specified by a parameter $p = (p_{\yhat a})_{\yhat,a}$, a vector of probabilities such that $p_{\yhat a} := \Ps \, [ \Yhat_p = 1 \, | \, \Yhat = \yhat, A = a ]$.  Among all fair $\Yhat_p$'s, the one with minimum error can be found by solving a linear program whose coefficients depend only on the $\qhat_{\yhat a y}$ quantities, and thus privacy will be achieved if these quantities are calculated privately using the Laplace mechanism. Once we do this, the differential privacy guarantees of the algorithm will follow from the post-processing property (Lemma \ref{postproc}). While the approach is straightforward and simply implementable, the privately learned classifier will need to explicitly take as input the sensitive attribute $A$ at test time which is not feasible (or legal) in all applications.

We have the \textbf{DP-postprocessing} algorithm written in Algorithm \ref{dpalgo1}. Notice as discussed above, to guarantee differential privacy of the protected attribute, Algorithm \ref{dpalgo1} computes $\qtilde_{\yhat a y}$ (a noisy version of $\qhat_{\yhat a y}$) and then feeds $\qtilde_{\yhat a y}$ into the linear program $\widetilde{\text{LP}}$ (\ref{lptilde}). In this linear program, terms with tildes (e.g. $\qtilde_{a y}$, $\errtilde$, $\fprtilde$, $\tprtilde$) are defined with respect to $\qtilde_{\yhat a y}$ instead of $\qhat_{\yhat a y}$. We analyze the performance of Algorithm \ref{dpalgo1} in Theorem \ref{theorem}. Its proof is deferred to Appendix \ref{subsec:postprocproof}. The main step of the proof is to understand how the introduced noise propagates to the solution of the linear program. We aill also briefly review the fair learning approach of \cite{hardt} in Appendix \ref{subsec:hardt}.
\medskip
\begin{tcolorbox}[title=$\widetilde{\text{LP}}$: $\epsilon$-Differentially Private Linear Program]
\begin{equation}\label{lptilde}
\begin{aligned}
& \ \ \ \argmin_{p} && \errtilde (\Yhat_p)  \\
& \text{ s.t. $ \forall \underset{a \neq 0}{a \in \A}$}   && \Delta \fprtilde_a (\Yhat_p )  \le \gamma + \frac{4 \ln \left(4|\A|/\beta \right) }{\min \{ \qtilde_{a0}, \qtilde_{00}\} \, m \epsilon } \\
&&& \Delta \tprtilde_a (\Yhat_p )  \le \gamma + \frac{4 \ln \left(4|\A|/\beta \right) }{\min \{ \qtilde_{a1}, \qtilde_{01}\} \, m \epsilon } \\
&&& 0 \le p_{\yhat a} \le 1 \quad \forall \yhat,a
\footnotetext[0]{$\errtilde \left(\Yhat_p \right) := \sum_{\yhat, a} \left( \qtilde_{\yhat a 0} - \qtilde_{\yhat a 1} \right) \cdot p_{\yhat a} + \sum_{\yhat,a} \qtilde_{\yhat a 1}$}
\footnotetext[0]{ $\Delta \fprtilde_a \left(\Yhat_p \right) := \Big\vert \fprtilde_a (\Yhat) \cdot p_{1a} + \left(1 - \fprtilde_a (\Yhat ) \right) \cdot p_{0a}  - \fprtilde_0 (\Yhat) \cdot p_{10} - \left(1 - \fprtilde_0 (\Yhat ) \right) \cdot p_{00} \Big\vert$}
\footnotetext[0]{ $\Delta \tprtilde_a \left(\Yhat_p \right)  := \Big\vert \tprtilde_a (\Yhat) \cdot p_{1a} + \left(1 - \tprtilde_a (\Yhat) \right) \cdot p_{0a} - \tprtilde_0 (\Yhat) \cdot p_{10} - \left(1 - \tprtilde_0 (\Yhat) \right) \cdot p_{00}  \Big\vert$}
\end{aligned}
\end{equation}
\end{tcolorbox}

\begin{algorithm}
\KwIn{privacy parameter $\epsilon$,  \\ \ \ \ \ \ \ \ \ \ \ \ confidence parameter $\beta$, fairness violation $\gamma$, \\ \ \ \ \ \ \ \ \ \ \ \ training examples $\{ (X_i, A_i, Y_i) \}_{i=1}^m$ }
\medskip
\begin{itemize}[label=\ding{212}]
\item Train the base classifier $\Yhat$ on $\{ (X_i, Y_i) \}_{i=1}^m$. \\
\item Calculate $\qhat_{\yhat a y} = \Pshat \, [ \Yhat = \yhat, A= a, Y= y ]$.\\
\item Sample $W_{\yhat a y} \overset{i.i.d.}{\sim} \text{Lap} \left(2/m \epsilon \right)$ for all $\yhat, a, y$. \\
\item Perturb each $\qhat_{\yhat a y}$: $\qtilde_{\yhat a y} = \qhat_{\yhat a y} + W_{\yhat a y}$. \\
\item Solve $\widetilde{\text{LP}}$ (\ref{lptilde}) to get the minimizer $\tilde{p}^\star$.
\end{itemize}
\vspace{0.25cm}
\KwOut{$\tilde{p}^\star$, the trained classifier $\Yhat$}
\caption{$\epsilon$-differentially private fair classification: \textbf{DP-postprocessing}}
\label{dpalgo1}
\end{algorithm}

\medskip

\begin{theorem}[Error-Privacy, Fairness-Privacy Tradeoffs]\label{theorem}
Suppose $\min\limits_{a,y} \{ \qhat_{ay}\} >4 \ln \left(4|\A|/\beta \right)/ \left( m \epsilon \right)$. Let $\widehat{p}^\star$ be the optimal $\gamma$-fair solution of the non-private post-processing algorithm of \cite{hardt} and let $\widetilde{p}^\star$ be the output of Algorithm~\ref{dpalgo1} which is the optimal solution of $\widetilde{\text{LP}}$ (\ref{lptilde}). With probability at least $1 - \beta$,
$$
\errhat \left(\Yhat_{\widetilde{p}^\star} \right) \le \errhat \left(\Yhat_{\widehat{p}^\star} \right) + \frac{24 |\A| \ln \left(4|\A|/\beta \right)}{m\epsilon}
$$
and for all $a \neq 0$,
$$
 \Delta \fprhat_a \left(\Yhat_{\widetilde{p}^\star} \right) \le  \gamma + \frac{8 \ln \left(4|\A|/\beta \right) }{\min \{ \qhat_{a0}, \qhat_{00}\} \, m \epsilon - 4 \ln \left(4|\A|/\beta \right)}
$$
$$
\Delta \tprhat_a \left(\Yhat_{\widetilde{p}^\star} \right)  \le  \gamma + \frac{8 \ln \left(4|\A|/\beta \right) }{\min \{ \qhat_{a1}, \qhat_{01}\} \, m \epsilon - 4 \ln \left(4|\A|/\beta \right)}
$$
\end{theorem}

We emphasize that the accuracy guarantee stated in Theorem \ref{theorem} is relative to the non-private post-processing algorithm, \emph{not} relative to the optimal fair classifier. This is because the non-private post-processing algorithm itself has no such optimality guarantees: its main virtue is simplicity. In the next section, we analyze a more complicated algorithm that is competitive with the optimal fair classifier. 


\section{Differentially Private Fair Learning: In-processing}\label{sec:inproc}
\begin{tcolorbox}[title= $\gamma$-fair ERM Problem]
\begin{equation}\label{box:fairerm}
\begin{aligned}
& \ \ \ \min_{Q \, \in \, \Delta(\Hs)}  & & \errhat (Q) \\
& \text{ s.t. $\forall \underset{a \neq 0}{a \in \A}$:} & & \Delta \fprhat_a (Q) \le \gamma \\
& & & \Delta \tprhat_a (Q) \le \gamma
\end{aligned}
\end{equation}
\end{tcolorbox}
In this section we will introduce our second differentially private fair learning algorithm which will be called \textbf{DP-oracle-learner} and is based on the algorithm presented in \cite{agarwal}. Essentially, \cite{agarwal} reduces the $\gamma$-fair learning problem (\ref{box:fairerm})  into the following Lagrangian min-max problem:
\begin{equation}\label{fairerm}
\min_{Q \, \in \, \Delta(\Hs)} \quad \max_{\lamb \, \in \, \Lambda} \quad L(Q, \lamb):= \errhat (Q) + \lamb^\top \rhat \, (Q)
\end{equation}
Here $\Hs$ is a given class of binary classifiers with $d_\Hs = VCD(\Hs) < \infty$ and $\Delta (\Hs)$ is the set of all randomized classifiers that can be obtained by functions in $\Hs$. $\rhat \, (Q)$ is a vector of fairness violations of the classifier $Q$ across groups, and $\lamb \in \Lambda = \{ \lamb: \ || \lamb ||_1 \le B\}$ is the dual variable where the bound $B$ is chosen to ensure convergence. In this work,
$$
\rhat (Q) := \begin{bmatrix} \fprhat_a (Q) - \fprhat_0 (Q) - \gamma \\ \fprhat_0 (Q) - \fprhat_a (Q) - \gamma \\ \tprhat_a (Q) - \tprhat_0 (Q) - \gamma \\ \tprhat_0 (Q) - \tprhat_a (Q) - \gamma \end{bmatrix}_{\underset{a \neq 0}{a \, \in \, \A}} \in \R^{4 (|\A| - 1)}
\quad , \quad
\lamb = \begin{bmatrix} \lambda_{(a,0,+)} \\  \lambda_{(a,0,-)} \\ \lambda_{(a,1,+)} \\ \lambda_{(a,1,-)} \end{bmatrix}_{\underset{a \neq 0}{a \, \in \, \A}} \in \R^{4 (|\A|-1)}
$$
The method developed by \cite{agarwal}, in the language of \cite{gerrymandering} gives a reduction from finding an optimal fair classifier to finding the equilibrium of a two-player zero-sum game played between a ``Learner" ($Q$-player) who needs to solve an unconstrained learning problem (given access to an efficient cost-sensitive classification oracle) and an ``Auditor" ($\lamb$-player) who finds fairness violations. In an iterative framework, having the learner play its best response and the auditor play a no-regret learning algorithm (we use exponentiated gradient descent, or ``multiplicative weights'') guarantees convergence of the average plays to the equilibrium (\cite{freund}).

In Algorithm \ref{dpalgo}, to make the above approach differentially private, Laplace mechanism is used by the Auditor when computing the gradients and we let the Learner run the exponential mechanism (or some other private learning oracle) to compute approximate best responses.  This is the differentially private equivalent of assuming access to a perfect oracle, as is done in \cite{agarwal,gerrymandering}. In practice, the exponential mechanism would be substituted for a computationally efficient private learner with heuristic accuracy guarantees. Subroutine \ref{privatebesth} reduces the Learner's best response problem to privately solving a cost sensitive classification problem solved with a private oracle $\text{CSC}_{\epsilon'}(\Hs)$. Here we sketch the main steps of analyzing Algorithm \ref{dpalgo}. All the proofs of this section, as well as a brief review of \cite{agarwal}'s approach for the fair learning problem without privacy constraints, will appear in Appendix \ref{inprocapp}.

We assume in this section that the VC dimension of $\Hs$ ($=d_{\Hs}$) is finite, in which case the set of strategies for the Learner reduces to $\Delta (\Hs (S))$, where $\Hs(S)$ is the set of all possible labellings induced on $S :=\{ X_i \}_{i=1}^m$ by $\Hs$. In other words, $\Hs(S) = \left\{ (h(X_1), \ldots, h(X_m)) \vert h \in \Hs\right\}$ and recall that $| \Hs(S) | \le O(m^{d_{\Hs}})$ by Sauer's Lemma. Note that since the privacy of the protected attribute $A$ is required, we need $A$ to be excluded from the domain of functions in $\Hs$ (``$A$-blind classification") and accordingly, from the set $S$. Because otherwise there might be some privacy loss of $A$ through using $\Hs(S)$ as the range of the exponential mechanism for the private Learner. This assumption is of course not necessary if one is willing to instead assume $|\Hs| < \infty$. We will have a discussion later where we state our guarantees assuming $|\Hs| < \infty$ instead of $d_\Hs < \infty$. Note that having $\Hs(S)$ as the range of the exponential mechanism used by the private Learner implies the privacy of the unprotected attributes $X$ is \emph{not} guaranteed. However, in the more general setting where $|\Hs| < \infty$ is assumed, the privacy of the unprotected attributes comes for free as there will be no reduction of $\Hs$ to $\Hs(S)$.


\begin{subroutine}
\KwIn{$\lamb$, training examples $\{ (X_i, A_i, Y_i) \}_{i=1}^m$, privacy guarantee $\epsilon'$}
\medskip
\For{$i=1, \ldots, m$}{
$C_i^0 \leftarrow \1 \{ Y_i \neq 0 \}$ \\
$C_i^1 \leftarrow \1 \{ Y_i \neq 1 \} + \frac{\lambda_{(A_i, Y_i, +)} - \lambda_{(A_i, Y_i, -)}}{\qhat_{A_i Y_i}} \1 \{ A_i \neq 0\} - \underset{\underset{a \neq 0}{a \in \A}}{\sum} \frac{\lambda_{(a, Y_i, +)} - \lambda_{(a, Y_i, -)}}{\qhat_{A_i Y_i}} \1 \{ A_i = 0\}$
}
Call $\text{CSC}_{\epsilon'}(\Hs)$ with $\{ X_i, C_i^0, C_i^1\}_{i=1}^m$ to get $h^\star$.
\medskip

\KwOut{$h^\star$}
\caption{$\besth^{\epsilon'}$}
\label{privatebesth}
\end{subroutine}

\begin{algorithm}
\KwIn{privacy parameters $(\epsilon, \delta)$, \\ \ \ \ \ \ \ \ \ \ \ \ bound $B$, VC dimension $d_{\Hs}$, confidence parameter $\beta$, fairness violation $\gamma$, \\ \ \ \ \ \ \ \ \ \ \ \ training examples $\left\{ (X_i, A_i, Y_i) \right\}_{i=1}^m$ }
\medskip
$$
T \leftarrow \frac{B \sqrt{\ln \left(4 |\A| -3 \right) } \, m \, \epsilon}{2 \left(2|\A|B+1 \right) \sqrt{\ln \left(1/\delta \right)} \left( d_{\Hs} \ln \left(m \right) + \ln \left(2 /\beta \right) \right)}, \quad \eta \leftarrow \frac{1}{2} \sqrt{\frac{\ln \left(4|\A|-3 \right)}{T}}
$$

$\thetatilde_1 \leftarrow \boldsymbol{0} \in \R^{4 (|\A|-1)}$\\
\For{$t=1, \ldots, T$}{
$\widetilde{\lambda}_{t,k} \leftarrow B \frac{\exp \, (\widetilde{\theta}_{t,k})}{1 + \sum_{k'} \exp \, (\widetilde{\theta}_{t,k'})}$ for $1 \le k \le 4 (|\A|-1)$ \\
$\htilde_t \leftarrow \besth^{\epsilon'} (\lambtilde_t)$ with $\epsilon' = \epsilon /(4\sqrt{T\ln(1/\delta)})$\\
Sample $\boldsymbol{W}_t \in \R^{4 (|\A|-1)}$  where $W_{t,k} \overset{i.i.d.}{\sim} \text{Lap} \, (\frac{8 |\A| \sqrt{T\ln(1/\delta)}}{(\min_{a,y} \{ \qhat_{ay} \} \, m - 1) \cdot \epsilon} )$\\
$\rtilde_t \leftarrow \rhat_t \, (\htilde_t) + \boldsymbol{W}_t$\\
$\thetatilde_{t+1} \leftarrow \thetatilde_{t} + \eta \rtilde_t$
}
$\Qtilde \leftarrow \frac{1}{T} \sum_{t=1}^T \htilde_t$, \quad $\lambtilde \leftarrow \frac{1}{T} \sum_{t=1}^T \lambtilde_t$
\medskip

\KwOut{$(\Qtilde, \lambtilde)$}
\caption{$(\epsilon, \delta)$-differentially private fair classification: \textbf{DP-oracle-learner}}
\label{dpalgo}
\end{algorithm}



\medskip
We first bound the regret of the Learner and the Auditor in Lemma \ref{regretq} and \ref{regretlambda} by understanding how the introduced noise affect these regrets. Proofs of these Lemmas follow from the ``sensitivity" and ``accuracy" of the private players which are all stated and proved in Appendix \ref{subsec:proofsinproc}.
\medskip
\begin{lemma}[Regret of the Private Learner]\label{regretq}
Suppose $\{ \htilde_t \}_{t=1}^T$ is the sequence of best responses to $\{ \lambtilde_t \}_{t=1}^T$ by the private Learner over $T$ rounds. We have that with probability at least $1 - \beta/2$,
\begin{equation*}\label{eq:regretq}
\frac{1}{T} \sum_{t=1}^T L(\htilde_t, \lambtilde_t) - \frac{1}{T} \min_{Q \in \Delta(\Hs)} \sum_{t=1}^T L(Q, \lambtilde_t) \, \le \, \frac{8 \left(2 |\A| B + 1 \right) \sqrt{T\ln \left(1/\delta \right)} \left( d_{\Hs} \ln \left(m \right) + \ln \left(2 T/\beta \right) \right)}{\left(\min_{a,y} \{ \qhat_{ay} \} \, m - 1 \right) \cdot \epsilon}
\end{equation*}
\end{lemma}

\begin{lemma}[Regret of the Private Auditor]\label{regretlambda}
Let $\{ \lambtilde_t \}_{t=1}^T$ be the sequence of exponentiated gradient descent plays (with learning rate $\eta$) by the private Auditor to given $\{ \htilde_t \}_{t=1}^T$ of the private Learner over $T$ rounds. We have that with probability at least $1-\beta/2$,
\begin{equation*}\label{eq:regretlambda}
\frac{1}{T}\max_{\lamb \in \Lambda}\sum_{t=1}^T L(\htilde_t, \lamb) - \frac{1}{T} \sum_{t=1}^T L(\htilde_t, \lambtilde_t) \, \le \, \frac{B \ln(4 |\A| - 3)}{\eta T} + 4\eta B \left( 1 + \frac{4 |\A| \sqrt{T\ln(1/\delta)} \ln(8T|\A| /\beta)}{\left(\min_{a,y} \{ \qhat_{ay} \} \, m - 1 \right) \cdot \epsilon} \right)^2
\end{equation*}
\end{lemma}
\medskip
Now in Theorem \ref{theorem1}, given the regret bounds of Lemma \ref{regretq} and \ref{regretlambda}, we can characterize the average plays of both players. This theorem provides a formal guarantee that the output $(\Qtilde, \lambtilde)$ of Algorithm \ref{dpalgo} forms a ``$\nu$-approximate equilibrium" of the game between the Learner and the Auditor (where $\nu$ is specified in the theorem). This property essentially means neither play would gain more than $\nu$ if they palyed an strategy other than the ones output by the Algorithm.
\medskip
\begin{theorem}\label{theorem1}
Let $(\Qtilde, \lambtilde)$ be the output of Algorithm~\ref{dpalgo}. We have that with probability at least $1 - \beta$, $(\Qtilde, \lambtilde)$ is a $\nu$-approximate solution of the game, i.e.,
\begin{align*}
&L(\Qtilde, \lambtilde) \, \le \, L(Q, \lambtilde) + \nu \quad \text{for all } Q \in \Delta (\Hs) \\
&L(\Qtilde, \lambtilde) \, \ge \, L(\Qtilde, \lamb) - \nu \quad \text{for all } \lamb \in  \Lambda
\end{align*}
and that
\begin{align*}
\nu = \widetilde{O} \left( \frac{B}{\min_{a,y} \{ \qhat_{ay} \}} \sqrt{\frac{|\A| \sqrt{ \ln \left(1/\delta \right)} \left( d_{\Hs} \ln (m) + \ln \left(1/\beta \right) \right)}{m \, \epsilon}} \right)
\end{align*}
where we hide further logarithmic dependence on $m$, $\epsilon$, and $|\A|$ under the $\widetilde{O}$ notation.
\end{theorem}
\medskip
We are now ready to conclude the \textbf{DP-oracle-learner} algorithm's analysis with the main theorem of this subsection that provides high probability bounds on the accuracy and fairness violation of the output $\Qtilde$ of Algorithm \ref{dpalgo}. These bounds can be viewed as revealing the inherent tradeoff between privacy of the algorithm and accuracy or fairness of the output classifier where a stronger privacy guarantee (i.e. smaller $\epsilon$ and $\delta$) will lead to weaker accuracy and fairness guarantees.
\medskip
\begin{theorem}[Error-Privacy, Fairness-Privacy Tradeoffs]\label{theorem2}
Let $(\Qtilde, \lambtilde)$ be the output of Algorithm~\ref{dpalgo} and let $Q^\star$ be the solution to the non-private $\gamma$-fair ERM problem \ref{box:fairerm}. We have that with probability at least $1-\beta$,
\begin{align*}
\errhat \, (\Qtilde) \, \le \, \errhat \, (Q^\star) + 2 \nu
\end{align*}
and for all $ a \neq 0$,
\begin{align*}
&\Delta \fprhat_a \, (\Qtilde) \, \le \, \gamma + \frac{1+2 \nu}{B} \\
&\Delta \tprhat_a \,  (\Qtilde) \, \le \, \gamma + \frac{1+2 \nu}{B}
\end{align*}
where
\begin{align*}
\nu = \widetilde{O} \left( \frac{B}{\min_{a,y} \{ \qhat_{ay} \}} \sqrt{\frac{|\A| \sqrt{ \ln(1/\delta)} \left( d_{\Hs} \ln (m) + \ln (1/\beta) \right)}{m \, \epsilon}} \right)
\end{align*}
\end{theorem}

\begin{remark}\label{btradeoff}
Notice the bounds stated above reveal a tradeoff between accuracy and fairness violation that we may control through the parameter $B$. As $B$ gets increased, the upper bound on error will get looser while the one on fairness violation gets tighter. We will consider a setting in the next subsection where we can remove this extra tradeoff and choose $B$ as small as possible --- at the cost of requiring that the classifiers be able to use protected attributes at test time.
\end{remark}
\medskip
We assumed so far in this section that the protected attribute $A$ is not available to the classifiers in $\Hs$ (``$A$-blind" classification) and stated all our bounds in terms of $d_\Hs$. In the more general setting where classifiers in $\Hs$ could depend on $A$ (``$A$-aware" classification), similar results hold. The only change to make is to replace $\ln \, (m^{d_\Hs})$ with $\ln \, (|\Hs|)$ in Algorithm \ref{dpalgo} (when computing the number of iterations $T$) and in the bounds. See Theorem \ref{theorem3} for this generalization.

\medskip
\begin{theorem}[Error-Privacy, Fairness-Privacy Tradeoffs]\label{theorem3}
Suppose $|\Hs| < \infty$ and let $(\Qtilde, \lambtilde)$ be the output of Algorithm~\ref{dpalgo} that runs for
$$
T =\frac{B \sqrt{\ln(4 |\A| -3) } \, m \, \epsilon}{2 \left(2|\A|B+1 \right) \sqrt{\ln \left(1/\delta \right)} \left( \ln \left(|\Hs| \right) + \ln \left(2 /\beta \right) \right)}
$$
iterations, and let $Q^\star$ be the solution to the non-private $\gamma$-fair ERM problem~\ref{box:fairerm}. We have that with probability at least $1-\beta$,
\begin{align*}
\errhat \, (\Qtilde) \, \le \, \errhat \, (Q^\star) + 2 \nu
\end{align*}
and for all $ a \neq 0$,
\begin{align*}
&\Delta \fprhat_a \, (\Qtilde) \, \le \, \gamma + \frac{1+2 \nu}{B} \\
&\Delta \tprhat_a \, (\Qtilde) \, \le \, \gamma + \frac{1+2 \nu}{B}
\end{align*}
where
\begin{align*}
\nu = \widetilde{O} \left( \frac{B}{\min_{a,y} \{ \qhat_{ay} \}} \sqrt{\frac{|\A| \sqrt{ \ln \left(1/\delta \right)} \left( \ln \left(|\Hs|/\beta \right) \right)}{m \, \epsilon}} \right)
\end{align*}
\end{theorem}

\subsection{An Extension: Better Tradeoffs for $A$-aware Classification}\label{sec:extension}
In this subsection we show that if we only ask for equalized false positive rates (instead of equalized odds, which also requires equalized true positive rates), and moreover, if we assume $\Hs$ includes all ``maximally discriminatory" classifiers (see Assumption \ref{ass5}), the fairness violation guarantees given in Theorem \ref{theorem3} can be improved. As a consequence, the tradeoff discussed in Remark \ref{btradeoff} will be no longer an issue. Thus, in this subsection, we are interested in solving the $\gamma$-fair ERM Problem \ref{fairerm2} which now only has false positive parity constraints.
\begin{tcolorbox}[title=$\gamma$-fair ERM Problem]
\begin{equation}\label{fairerm2}
\begin{aligned}
& \ \ \ \min_{Q \, \in \, \Delta(\Hs)}  & & \errhat (Q) \\
& \text{ s.t. $\forall \underset{a \neq 0}{a \in \A}$:} & & \Delta \fprhat_a (Q) \le \gamma
\end{aligned}
\end{equation}
\end{tcolorbox}
\medskip
\begin{assumption}\label{ass5}
$\Hs$ includes all maximally discriminatory classifiers (i.e. group indicator functions): $\{ h_a (X,A) = \1_{A=a}, \, \bar{h}_a(X,A) = \1_{A \neq a} \, \vert \, a \in \A \} \subseteq \Hs$.
\end{assumption}
\medskip
\begin{theorem}[Error-Privacy, Fairness-Privacy Tradeoffs]\label{theorem4}
Suppose $|\Hs| < \infty$, $ B > |A| - 1$, and let Assumption \ref{ass5} hold. Let $(\Qtilde, \lambtilde)$ be the output of Algorithm~\ref{dpalgo}, and let $Q^\star$ be the solution to the $\gamma$-fair ERM problem~\ref{fairerm2}. We have that with probability at least $1-\beta$,
\begin{align*}
\errhat \, (\Qtilde) \, \le \, \errhat \, (Q^\star) + 2 \nu
\end{align*}
and for all $ a \neq 0$,
\begin{align*}
&\Delta \fprhat_a \, (\Qtilde) \, \le \, \gamma + \frac{2\nu}{B - (|\A|-1)}
\end{align*}
\end{theorem}

As an immediate consequence of Theorem \ref{theorem4}, we have the following Corollary where $B = |\A|$ can be chosen to get bounds which are now free of $B$.
\medskip
\begin{corollary}
Under assumptions stated in Theorem \ref{theorem4}, one can choose $B=|\A|$ in Algorithm \ref{dpalgo}, in which case with probability at least $1-\beta$,
\begin{align*}
\errhat \, (\Qtilde) \, \le \, \errhat \, (Q^\star) + 2 \nu
\end{align*}
and for all $ a \neq 0$,
\begin{align*}
&\Delta \fprhat_a \, (\Qtilde) \, \le \, \gamma + 2 \nu
\end{align*}
where
\begin{align*}
\nu = \widetilde{O} \left( \frac{|\A|}{\min_{a,y} \{ \qhat_{ay} \}} \sqrt{\frac{|\A| \sqrt{ \ln(1/\delta)} \ln (|\Hs|/\beta) }{m \, \epsilon}} \right)
\end{align*}
\end{corollary}


\subsection{A Separation: $A$-blind vs. $A$-aware Classification}
\label{sec:separation}
In this subsection we show that the sensitivity of the accuracy of the optimal classifier subject to fairness constraints can be substantially higher if it is prohibited from using sensitive attributes at test time. This implies that higher error must be introduced when estimating this accuracy subject to differential privacy. This shows a fundamental tension between the goals of trading off privacy and approximate equalized odds, with the goal of preventing disparate treatment. 
 Given a data set $D$ of $m$ individuals, define $f(D)$ to be the optimal error rate in the $\gamma$-fair ERM problem \ref{fairerm2} which is constrained to have a false positive rate disparity of at most $\gamma$. 

Consider the following problem instance. Let $X$ be the unprotected attribute taking value in $\X = \{U,V\}$, and let $A$ be the protected attribute taking value in $\A = \{R,B\}$. Suppose $\Hs$ consists of two classifiers $h_0$ and $h_U$ where $h_0(X,A) = 0$ and $h_U(X,A) = \1_{X=U}$. Notice that both $h_0$ and $h_U$ depend only on the unprotected attribute. Consider two other classifiers $h_R$ and $h_B$ that depend on the protected attribute: $h_R(X,A) = \1_{A=R}$ and $h_B(X,A) = \1_{A=B}$.
\medskip
\begin{theorem}\label{thm:separation}
Consider $\gamma > 1/m$ and data sets with $\min_a \qhat_{a0} \geq C$ for some constant $C > 0$. If $\Hs = \{h_0, h_U\}$, the sensitivity of $f$ is $\Omega(1/(\gamma m))$. If the ``maximally discriminatory" classifier $h_R$ and $h_B$ are included in $\Hs$ as well, i.e. $\Hs = \{h_0,h_U,h_R,h_B\}$, the sensitivity of $f$ is $O(1/ m)$.
\end{theorem}

\section{Experimental Evaluation}

\begin{figure}
\centering
\includegraphics[scale = 0.55]{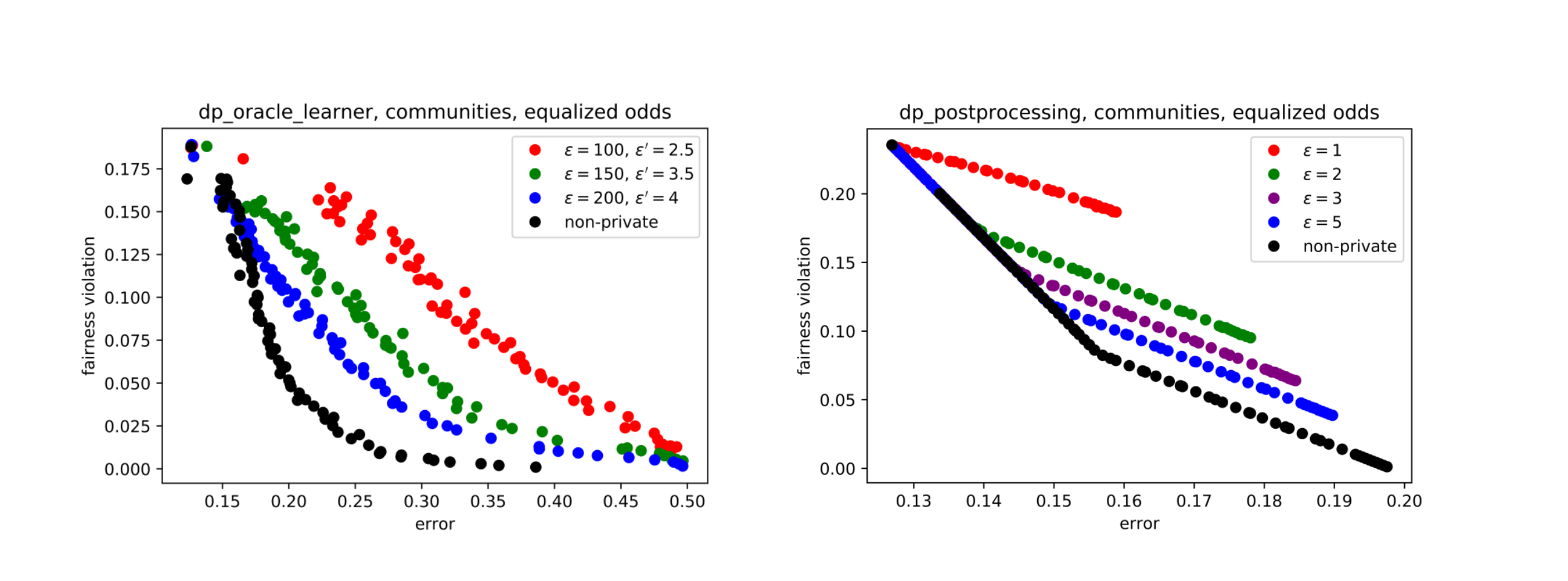}
\caption{Left figure shows the Pareto frontier of error and (equalized odds) fairness violation for the DP-oracle-learner algorithm on the Communities dataset across different privacy parameters . Right figure shows the corresponding Pareto curves for the DP-postprocessing algorithm. Each point on the private curves is averaged over many rounds to reduce the effect of noise variance. See text for details.}
\label{fig:experiment}
\end{figure}

As a proof of concept, we empirically evaluate our two algorithms on a common fairness benchmark dataset: 
the Communities and Crime dataset\footnote{Briefly, each record in this dataset 
summarizes aggregate socioeconomic
information about both the citizens and police force in a particular U.S. community, and the problem is to predict
whether the community has a high rate of violent crime.}
from the UC Irvine Machine Learning Repository. We refer the reader to \cite{empgerry} for an outline of potential fairness concerns present in the dataset. We clean and preprocess the data identically to \cite{empgerry}. Our main experimental goal is to obtain, for both algorithms, the Pareto frontier of error and fairness violation tradeoffs for different levels of differential privacy. To elaborate, for a given setting of input parameters, we start with the target fairness violation bound $\gamma = 0$ and then increase it over a rich pre-specified subset of $[0,1]$ while recording for each $\gamma$ the error and the (realized) fairness violation of the classifier output by the algorithm. We take $\Hs$ to be the class of linear threshold functions, $\beta = 0.05$, and $\delta = 10^{-7}$.

Logistic regression is used as the base classifier of the \textbf{DP-postprocessing} algorithm in our experiments. To implement the Learner's cost-sensitive classification oracle used in the \textbf{DP-oracle-learner} algorithm, following \cite{empgerry}, we build a regression-based linear predictor for each vector of costs ($C_0$ and $C_1$), and classify a point according to the lowest predicted cost. We made this private following the method of \cite{smith2017interaction}: computing each regression as $(X^T X)^{-1} X^T C_{b}$, and adding appropriately scaled Laplace noise to both $X^T X$ and $X^T C_{b}$. Note when the sensitive attribute $A$ is not included in $X$ (the $A$-blind case, as in our experiments) noise need not be added to $X^T X$ as we only need to guarantee the privacy of $A$.

The theory is ambiguous in its predictions about which algorithm should perform better: the ``privacy cost'' is higher for the in-processing algorithm, but the benchmark that the post-processing algorithm competes with is weaker. We would generally expect therefore that on sufficiently large datasets, the in-processing algorithm would obtain better tradeoffs, but on small datasets, the post-processing algorithm would.

Our experimental results appear in Fig. \ref{fig:experiment}. Indeed, on our relatively small dataset ($m \approx 2$K), the post-processing algorithm can obtain good tradeoffs between accuracy and fairness at meaningful levels of $\epsilon$, whereas the in-processing algorithm cannot. Nevertheless, we can empirically obtain the ``shape'' of the Pareto curve trading off accuracy and fairness for unreasonable levels of $\epsilon$ using our algorithm. This is still valuable, because the value of $\epsilon$ obtained by our algorithms predictably decreases as the dataset size $m$ increases without otherwise changing the dynamics of the algorithm. For example, if we ``upsampled'' our dataset by a factor of 10 (i.e. taking 10 copies of the dataset), the result would be a reasonably sized dataset of $m \approx 20$K. Our algorithm run on this upsampled dataset would obtain the same tradeoff curve but now with meaningful values of $\epsilon$. In the left panel of Fig. \ref{fig:experiment}, $\epsilon$ is the actual privacy parameter used in the experiments; while $\epsilon'$ is the value that the privacy parameter would take on the dataset that was upsampled by a factor of 10. 

Recall that the post-processing approach requires the use of the protected attribute at test time, but the in-processing approach does not. Our results therefore suggest that the requirement that we {\em not\/} use the protected attribute at test time (i.e. that we be avoid ``disparate treatment'') might be extremely burdensome if we also want the protections of differential privacy and have only small dataset sizes. In contrast, it can be overcome with the in-processing algorithm at larger dataset sizes.

\section*{Acknowledgements}
AR is supported in part by NSF grants AF-1763307 and CNS-1253345. JU is supported by NSF grants CCF-1718088, CCF-1750640, and CNS-1816028, and a Google Faculty Research Award.

\bibliographystyle{apalike}
\bibliography{fairdp}

\appendix

\section{Appendix for Models and Preliminaries: Differential Privacy}
\label{privacyapp}
An important property of differential privacy is that it is robust to \textit{post-processing}. The post-processing of an $(\epsilon,\delta)$-DP algorithm output remains $(\epsilon,\delta)$-DP.
\medskip
\begin{lemma}[Post-Processing \cite{dwork}]\label{postproc}
Let $M: \D^m \to \mathcal{O}$ be a $(\epsilon, \delta)$-DP algorithm and let $f: \mathcal{O} \to \mathcal{R}$ be any randomized function. We have that the algorithm $f \, o \, M: \D^m \to \mathcal{R}$ is $(\epsilon, \delta)$-DP.
\end{lemma}

Another important property of differential privacy is that DP algorithms can be composed adaptively with a graceful degradation in their privacy parameters.
\medskip
\begin{theorem}[Composition \cite{DworkRV10}]\label{composition}
Let $M_t$ be an $(\epsilon_t, \delta_t)$-DP algorithm for $t \in [T]$. We have that the composition $M = (M_1, \ldots, M_T)$ is $(\epsilon, \delta)$-DP where $\epsilon = \sum_{t} \epsilon_t$ and $\delta = \sum_t \delta_t$.
\end{theorem}
Following the Composition Theorem~\ref{composition}, if for instance, an iterative algorithm that runs in $T$ iterations is to be made private with target privacy parameters $\epsilon$ and $\delta = 0$, each iteration must be made $\epsilon/T$-DP. This may lead to a huge amount of per iteration noise if $T$ is too large. The Advanced Composition Theorem \ref{advcomposition} instead allows the privacy parameter at each step to scale with $O(\epsilon/\sqrt{T})$.
\medskip
\begin{theorem}[Advanced Composition \cite{DworkRV10}]\label{advcomposition}
Suppose $0 < \epsilon < 1$ and $\delta > 0$ are target privacy parameters. Let $M_t$ be a $(\epsilon', \delta')$-DP algorithm for all $t \in [T]$. We have that the composition $M = (M_1, \ldots, M_T)$ is $(\epsilon, T\delta' + \delta)$-DP where $\epsilon = 2 \epsilon' \sqrt{2 T \ln(1/\delta)}$.
\end{theorem}

\section{Appendix for DP Fair Learning: Post-processing}\label{postprocapp}

\subsection{Fair Learning Approach of \cite{hardt}}\label{subsec:hardt}
We briefly review the fair learning approach of \cite{hardt} in this subsection. Suppose there is an arbitrary base classifier $\Yhat$ which is trained on the set of training examples $\{ \left( X_i, Y_i \right) \}_{i=1}^m$. The goal is to make the classifications of the base classifier $\gamma$-fair with respect to the sensitive attribute $A$ by post-processing the predictions given by $\Yhat$. With slight abuse of notation, let $\Yhat_p$ denote the derived optimal $\gamma$-fair randomized classifier where $p = (p_{\yhat a})_{\yhat,a}$ is a vector of probabilities describing $\Yhat_p$ and that $p_{\yhat a} := \Ps \, [ \Yhat_p = 1 \, | \, \Yhat = \yhat, A = a ]$. Among all fair $\Yhat_p$'s, the one with minimum error can be found by solving the optimization problem $\text{LP}$ (\ref{lp}). Once the optimal solution $p^\star$ is found, one would then use this vector of probabilities, along with the estimate $\Yhat$ given by the base classifier and the sensitive attribute $A$, to make further predictions. See Fig.~\ref{model} for a visual presentation of the adopted model.
\begin{figure}
\centering
\includegraphics[scale = 0.75]{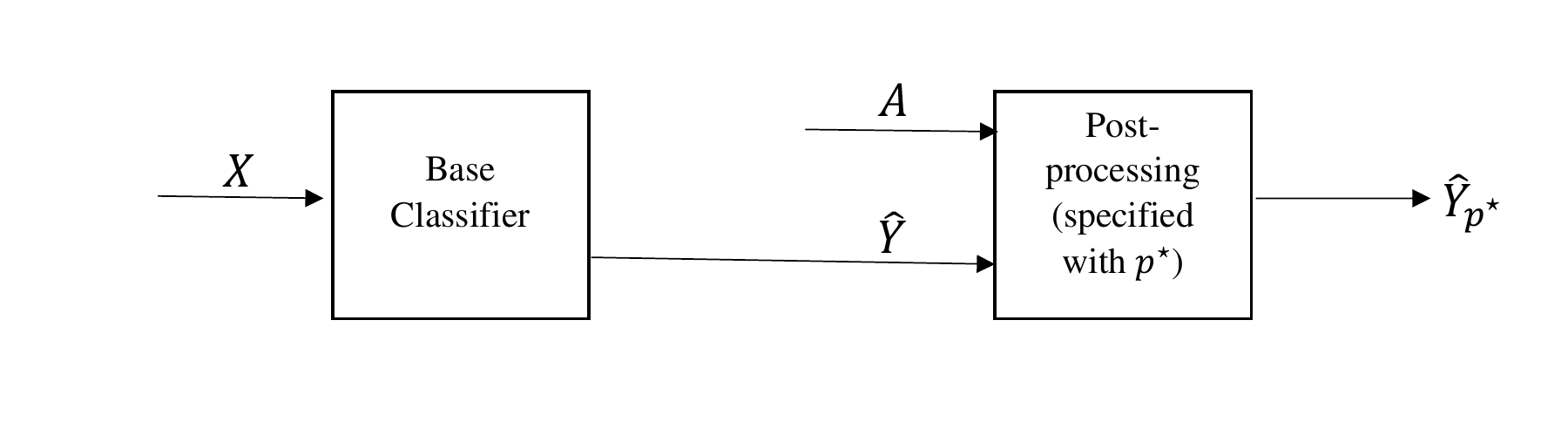}
\caption{The post-processing technique. In the training phase, training examples are used to train the base classifier and find the optimal $p^\star$ by solving $\text{LP}$ (\ref{lp}).}
\label{model}
\end{figure}
\begin{tcolorbox}[title=$\text{LP}$: Linear Program]
\begin{equation}\label{lp}
\begin{aligned}
& \ \ \ \argmin_{p}  & & \err \left(\Yhat_p \right) \\
& \text{ s.t. $\forall \underset{a \neq 0}{a \in \A}$} & & \Delta \fpr_a \left(\Yhat_p \right)  \le \gamma \\
& & &  \Delta \tpr_a \left(\Yhat_p \right)  \le \gamma \\
& & & 0 \le p_{\yhat a} \le 1 \quad \forall \yhat,a
\end{aligned}
\end{equation}
\end{tcolorbox}
Since the true underlying distribution $\Px$ is not known, in practice the empirical distribution $\Pxhat$ is used to estimate the quantities appearing in $\text{LP}$ (\ref{lp}). Using simple probability techniques, one can expand the empirical quantities $\errhat(\Yhat_p)$, $\Delta \fprhat_a (\Yhat_p)$, and $\Delta \tprhat_a (\Yhat_p)$ in a linear form in $p$ with coefficients being a function of $\qhat_{\yhat a y}$ and $\qhat_{a y}$ quantities (see $\widehat{\text{LP}}$ (\ref{lphat})).
\begin{tcolorbox}[title=$\widehat{\text{LP}}$: Empirical Linear Program]
\begin{equation}\label{lphat}
\begin{aligned}
& \ \ \  \argmin_{p} && \errhat \left(\Yhat_p \right) \\
& \text{ s.t. $ \forall \underset{a \neq 0}{a \in \A}$}   && \Delta \fprhat_a \left(\Yhat_p \right) \le \gamma \\
&&&  \Delta \tprhat_a \left(\Yhat_p \right)  \le \gamma \\
&&& 0 \le p_{\yhat a} \le 1 \quad \forall \yhat,a
\footnotetext[0]{$\errhat \left(\Yhat_p \right) = \sum_{\yhat, a} \left( \qhat_{\yhat a 0} - \qhat_{\yhat a 1} \right) \cdot p_{\yhat a} + \sum_{\yhat,a} \qhat_{\yhat a 1}$}
\footnotetext[0]{$\Delta \fprhat_a \left(\Yhat_p \right) = \Big\vert \fprhat_a \left(\Yhat \right) \cdot p_{1a} + \left(1 - \fprhat_a \left(\Yhat \right) \right) \cdot p_{0a} - \fprhat_0 \left(\Yhat \right) \cdot p_{10} - \left(1 - \fprhat_0 \left(\Yhat \right) \right) \cdot p_{00} \Big\vert$}
\footnotetext[0]{$\Delta \tprhat_a \left(\Yhat_p \right) = \Big\vert \tprhat_a \left(\Yhat \right) \cdot p_{1a} + \left(1 - \tprhat_1 \left(\Yhat \right) \right) \cdot p_{0a}  - \tprhat_0 \left(\Yhat \right) \cdot p_{10} - \left(1 - \tprhat_0 \left(\Yhat \right) \right) \cdot p_{00}  \Big\vert$}
\end{aligned}
\end{equation}
\end{tcolorbox}

\subsection{Proof of Theorem \ref{theorem}}\label{subsec:postprocproof}
The proof of Theorem \ref{theorem} relies on some facts which are stated here.
\medskip
\begin{claim}[$\ell_1$-Sensitivity of $\boldsymbol{\qhat}$ to $A$]\label{sensitivity1}
Let $\boldsymbol{\qhat} = \left[ \qhat_{\yhat a y} \right]_{\yhat, a, y}$ be the empirical distribution of $\{ \Yhat, A, Y\}$ and let $\Delta \boldsymbol{\qhat}$ be the $\ell_1$-sensitivity of $\boldsymbol{\qhat}$ to $A$. We have that
$$
\Delta \boldsymbol{\qhat} = \max_{\overset{A,A' \, \in \, \A^m}{A \sim A'}} \left\| \, \boldsymbol{\qhat} \left(A \right) - \boldsymbol{\qhat} \left(A' \right)  \right\|_1 = \frac{2}{m}
$$
\end{claim}
\medskip
\begin{lemma}\label{lemma2}
Suppose $\min\limits_{a,y} \{ \qhat_{ay}\} >4 \ln \left(4|\A|/\beta \right)/ \left( m \, \epsilon \right)$. we have that with probability $\ge 1 - \beta$,
\begin{enumerate}
\item $\left\vert \errtilde \left(\Yhat_p \right) - \errhat \left(\Yhat_p \right) \right\vert \le \frac{12 |\A| \ln \left(4|\A|/\beta \right)}{m\epsilon} \quad ; \forall \, p.$

\item $ \qtilde_{a y} > 0 \quad ; \forall \, a,y$.

\item $\left\vert \fprtilde_a \left(\Yhat \right) - \fprhat_a \left(\Yhat \right) \right\vert \le \frac{2 \ln \left(4|\A|/\beta \right) }{\qtilde_{a 0}  \, m  \epsilon}$, \quad  $\left\vert \tprtilde_a \left(\Yhat \right) - \tprhat_a \left(\Yhat \right) \right\vert \le \frac{2 \ln \left(4|\A|/\beta \right) }{ \qtilde_{a 1}  \, m \epsilon}  \quad ; \forall \, a.$

\item $\left\vert \Delta \fprtilde_a \left(\Yhat_p \right) - \Delta \fprhat_a \left(\Yhat_p \right) \right\vert \le  \frac{4 \ln \left(4|\A|/\beta \right) }{\min \{ \qtilde_{a0}, \qtilde_{00}\} \, m \epsilon}, \, \left\vert \Delta \tprtilde_a \left(\Yhat_p \right) - \Delta \tprhat_a \left(\Yhat_p \right) \right\vert \le  \frac{4 \ln \left(4|\A|/\beta \right) }{\min \{ \qtilde_{a1}, \qtilde_{01}\} \, m \epsilon} \ ; \forall \, a,p.$

\item $\hat{p}^\star$, the optimal solution of $\widehat{\text{LP}}$ (\ref{lphat}), is feasible in $\widetilde{\text{LP}}$ (\ref{lptilde}).

\end{enumerate}
\end{lemma}

\begin{proof}[Proof of Lemma~\ref{lemma2}] By Claim~\ref{sensitivity1} and Theorem \ref{laplacethm}, we have that with probability at least $1 - \beta$, $|| \boldsymbol{\qhat} - \boldsymbol{\qtilde} ||_{\infty} \le \ln \left( 4|\A| / \beta \right) \cdot \left( 2/m\epsilon \right)$. Hence with probability $\ge 1 - \beta$,
\begin{enumerate}

\item $\forall \, p$,
$$
\left\vert \errtilde (\Yhat_p) - \errhat (\Yhat_p) \right\vert \le \sum_{\yhat,a,y} \left\vert \qtilde_{\yhat a y} - \qhat_{\yhat a y}\right\vert + \sum_{\yhat,a} \left\vert \qtilde_{\yhat a 1} - \qhat_{\yhat a 1}\right\vert \le \frac{12 |\A| \ln (4|\A|/\beta)}{m\epsilon}
$$

\item For all $a,y$,
\begin{align*}
\left\vert \qtilde_{ay} - \qhat_{ay} \right\vert &= \left\vert \qtilde_{1ay} + \qtilde_{0ay} - \qhat_{1ay} - \qhat_{0ay}\right\vert \\
& \le \left\vert \qtilde_{1ay} - \qhat_{1ay} \right\vert + \left\vert \qtilde_{0ay} - \qhat_{0ay} \right\vert \\
& \le \frac{4 \ln (4|\A|/\beta) }{ m \epsilon}
\end{align*}
But by the stated assumption, $\qhat_{ay} > \frac{4 \ln (4|\A|/\beta) }{ m \epsilon}$ implying that $\qtilde_{ay} > 0$.

\item $\forall \, a$,
\begin{align*}
\left\vert \fprtilde_a (\Yhat) - \fprhat_a (\Yhat) \right\vert &= \left\vert \frac{\qtilde_{1a0}}{\qtilde_{1a0} + \qtilde_{0a0}} - \frac{\qhat_{1a0}}{\qhat_{1a0} + \qhat_{0a0}}\right\vert \\
&= \left\vert \frac{\qtilde_{1a0} \, \qhat_{0a0} - \qhat_{1a0} \, \qtilde_{0a0}}{(\qtilde_{1a0} + \qtilde_{0a0})(\qhat_{1a0} + \qhat_{0a0})}\right\vert \\
&= \left\vert \frac{\qhat_{0a0}(\qtilde_{1a0} - \qhat_{1a0}) - \qhat_{1a0} (\qtilde_{0a0} - \qhat_{0a0})}{(\qtilde_{1a0} + \qtilde_{0a0})(\qhat_{1a0} + \qhat_{0a0})}\right\vert \\
&\le \frac{2 \ln (4|\A|/\beta) }{|\qtilde_{a 0}| \, m \epsilon} \\
&= \frac{2 \ln (4|\A|/\beta) }{\qtilde_{a 0} \, m \epsilon} \quad (\text{by Part 2 of this Lemma})
\end{align*}
And similarly,
$$
\left\vert \tprtilde_a (\Yhat) - \tprhat_a (\Yhat) \right\vert \le \frac{2 \ln (4|\A|/\beta) }{\qtilde_{a 1} \, m \epsilon}
$$

\item Observe that $\forall \, a, p$,
\begin{align*}
&\left \vert \Delta \fprtilde_a (\Yhat_p) - \Delta \fprhat_a (\Yhat_p) \right \vert \\
&\le \Big\vert \fprtilde_a(\Yhat) \cdot p_{1a} + (1 - \fprtilde_a(\Yhat)) \cdot p_{0a} - \fprtilde_0(\Yhat) \cdot p_{10} - (1 - \fprtilde_0(\Yhat)) \cdot p_{00} \\
& \ \ - \fprhat_a(\Yhat) \cdot p_{1a} - (1 - \fprhat_a(\Yhat)) \cdot p_{0a} + \fprhat_0(\Yhat) \cdot p_{10} + (1 - \fprhat_0(\Yhat)) \cdot p_{00} \Big\vert \\
& \le \left\vert \fprtilde_a(\Yhat) - \fprhat_a(\Yhat) \right\vert \cdot \left\vert p_{1a} - p_{0a}\right\vert + \left\vert \fprtilde_0(\Yhat) - \fprhat_0(\Yhat) \right\vert \cdot \left\vert p_{10} - p_{0}\right\vert \\
& \le  \frac{4 \ln (4|\A|/\beta) }{\min \{ \qtilde_{a0}, \qtilde_{00}\} \, m \epsilon} \quad \textrm{(by part 3 of this Lemma)}
\end{align*}
A similar argument holds for $\left \vert \Delta \tprtilde_a (\Yhat_p) - \Delta \tprhat_a (\Yhat_p) \right \vert \le  \frac{4 \ln (4|\A|/\beta) }{\min \{ \qtilde_{a1}, \qtilde_{01}\} \, m \epsilon}$.

\item  We will show that $\widehat{p}^\star$ satisfies the first constraint of $\widetilde{\text{LP}}$ (\ref{lptilde}) for all $a \in \A$. Satisfying the second constraint can be similarly shown and the third is trivial. We have that
\begin{align*}
\left\vert \Delta \fprtilde_a (\Yhat_{\widehat{p}^\star}) \right \vert &= \left\vert  \Delta \fprtilde_a (\Yhat_{\widehat{p}^\star}) - \Delta \fprhat_a (\Yhat_{\widehat{p}^\star}) + \Delta \fprhat_a (\Yhat_{\widehat{p}^\star}) \right\vert \\
&\le \left\vert \Delta \fprhat_a (\Yhat_{\widehat{p}^\star}) \right\vert + \left\vert \Delta \fprtilde_a (\Yhat_{\widehat{p}^\star}) - \Delta \fprhat_a (\Yhat_{\widehat{p}^\star}) \right\vert \\
&\le \gamma + \frac{4 \ln (4|\A|/\beta) }{\min \{ \qtilde_{a0}, \qtilde_{00}\} \, m \epsilon}
\end{align*}
by part 4 of this Lemma and the fact that $\left\vert \Delta \fprhat_a (\Yhat_{\widehat{p}^\star}) \right\vert \le \gamma$ (see $\widehat{\text{LP}}$ (\ref{lphat})).
\end{enumerate}
\end{proof}

\begin{proof}[Proof of Theorem~\ref{theorem}]  Following Lemma~\ref{lemma2}, with probability at least $1 - \beta$
\begin{align*}
\errhat (\Yhat_{\widetilde{p}^\star}) &\le \errtilde (\Yhat_{\widetilde{p}^\star}) + \frac{12 |\A| \ln (4|\A|/\beta)}{m\epsilon} \quad \text{(part 1 of Lemma~\ref{lemma2})}\\
&\le \errtilde (\Yhat_{\widehat{p}^\star}) + \frac{12 |\A| \ln (4|\A|/\beta)}{m\epsilon} \quad \text{(part 5 of Lemma~\ref{lemma2})} \\
&\le \errhat (\Yhat_{\widehat{p}^\star}) + \frac{24 |\A| \ln (4|\A|/\beta)}{m\epsilon} \quad \text{(part 1 of Lemma~\ref{lemma2})}
\end{align*}
Also, for all $a \neq 0$,
\begin{align*}
\Delta \fprhat_a \, (\Yhat_{\widetilde{p}^\star}) & \, \le \, \Delta \fprtilde_a (\Yhat_{\widetilde{p}^\star}) + \frac{4 \ln (4|\A|/\beta) }{\min \{ \qtilde_{a0}, \qtilde_{00}\} \, m \epsilon}  \quad \text{(part 4 of Lemma~\ref{lemma2})} \\
& \le \, \gamma + \frac{8 \ln (4|\A|/\beta) }{\min \{ \qtilde_{a0}, \qtilde_{00}\} \, m \epsilon} \quad \text{(see $\widetilde{\text{LP}}$ (\ref{lptilde}))}  \\
& \le \, \gamma +  \frac{8 \ln (4|\A|/\beta) }{\min \{ \qhat_{a0}, \qhat_{00}\} \, m \epsilon - 4 \ln (4|\A|/\beta)}
\end{align*}
The last inequality follows from the fact that $| \qtilde_{a y} - \qhat_{a y}| \le 4 \ln (4|\A|/\beta) / m \epsilon$ for all $a,y$. It follows similarly that,
$$
\Delta \tprhat_a \, (\Yhat_{\widetilde{p}^\star}) \le  \gamma +  \frac{8 \ln (4|\A|/\beta) }{\min \{ \qhat_{a1}, \qhat_{01}\} \, m \epsilon - 4 \ln (4|\A|/\beta)}
$$
\end{proof}

\section{Appendix for DP Fair Learning: In-processing}\label{inprocapp}

\subsection{Fair Learning Approach of \cite{agarwal}}\label{subsec:agarwal}
Suppose given a class of binary classifiers $\Hs$, the task is to find the optimal $\gamma$-fair classifier in $\Delta (\Hs)$, where $\Delta (\Hs)$ is the set of all randomized classifiers that can be obtained by functions in $\Hs$. \cite{agarwal} provided a reduction of the learning problem with only the fairness constraint to a two-player zero-sum game and introduced an algorithm that achieves the lowest empirical error. In this section we mainly discuss their reduction approach which forms the basis of our differentially private fair learning algorithm: \textbf{DP-oracle-learner}. Although \cite{agarwal} considers a general form of a constraint that captures many existing notions of fairness, in this paper, we focus on the Equalized Odds notion of fairness described in Definition~\ref{eo}. Our techniques, however, generalize beyond this. To begin with, the $\gamma$-fair classification task can be modeled as the constrained optimization problem~\ref{fair}.
\begin{tcolorbox}[title=$\gamma$-fair Learning Problem]
\renewcommand{\thempfootnote}{\arabic{mpfootnote}}
\begin{equation}\label{fair}
\begin{aligned}
& \ \ \ \min_{Q \, \in \, \Delta(\Hs)}  & & \err (Q) \\
& \text{ s.t. $\forall \underset{a \neq 0}{a \in \A}$:} & & \Delta \fpr_a (Q) \le \gamma  \\
& & & \Delta \tpr_a (Q) \le \gamma
\end{aligned}
\end{equation}
\end{tcolorbox}
As the data generating distribution $\Px$ is unknown, we will be dealing with the Fair Empirical Risk Minimization (ERM) problem ~\ref{fairermapp}. In this empirical version, all the probabilities and expectations are taken with respect to the empirical distribution of the data $\Pxhat$.
\begin{tcolorbox}[title=$\gamma$-fair ERM Problem]
\begin{equation}\label{fairermapp}
\begin{aligned}
& \ \ \ \min_{Q \, \in \, \Delta(\Hs)}  & & \errhat (Q) \\
& \text{ s.t. $\forall \underset{a \neq 0}{a \in \A}$:} & & \Delta \fprhat_a (Q) \le \gamma \\
& & & \Delta \tprhat_a (Q) \le \gamma
\end{aligned}
\end{equation}
\end{tcolorbox}
Toward deriving a fair classification algorithm, the above fair ERM problem \ref{fairermapp} will be rewritten as a two-player zero-sum game whose equilibrium is the solution to the problem. Let $\rhat (Q) \in \R^{4 (|\A| - 1)}$ store all fairness violations of the classifier $Q$.
\begin{align*}
\rhat (Q) := \begin{bmatrix} \fprhat_a (Q) - \fprhat_0 (Q) - \gamma \\ \fprhat_0 (Q) - \fprhat_a (Q) - \gamma \\ \tprhat_a (Q) - \tprhat_0 (Q) - \gamma \\ \tprhat_0 (Q) - \tprhat_a (Q) - \gamma \end{bmatrix}_{\underset{a \neq 0}{a \in \A}} \in \R^{4 (|\A| - 1)}
\end{align*}
For dual variable $\lamb = \begin{bmatrix} \lambda_{(a,0,+)}, \,  \lambda_{(a,0,-)}, \, \lambda_{(a,1,+)}, \, \lambda_{(a,1,-)} \end{bmatrix}^\top_{\underset{a \neq 0}{a \in \A}} \in \R^{4 (|\A|-1)}$, let
\begin{align*}
L(Q, \lamb) = \errhat (Q) + \lamb^\top \rhat (Q)
\end{align*}
be the Lagrangian of the optimization problem. We therefore have that the Fair ERM Problem \ref{fairermapp} is equivalent to
$$
\min_{Q \, \in \, \Delta(\Hs)} \quad \max_{\lamb \, \in \, \R_+^{4|\A|}} \quad L(Q, \lamb)
$$
In order to guarantee convergence, we further constrain the $\ell_1$ norm of $\lamb$ to be bounded. So let $\Lambda = \{ \lamb \in \R_+^{4 (|\A| - 1)}: \ || \lamb ||_1 \le B\}$ be the feasible space of the dual variable $\lamb$ for some constant $B$. Hence, the primal and the dual problems are as follows.
\begin{align*}
&\text{primal problem:} \quad \min_{Q \, \in \, \Delta(\Hs)} \quad \max_{\lamb \, \in \, \Lambda} \quad L(Q, \lamb) \\
&\text{dual problem:} \quad \max_{\lamb \, \in \, \Lambda} \quad \min_{Q \, \in \, \Delta(\Hs)} \quad L(Q, \lamb)
\end{align*}
The above primal and dual problems can be shown to have solutions that coincide at a point $(Q^\star, \lamb^\star)$ which is the saddle point of $L$. From a game theoretic perspective, the saddle point can be viewed as an equilibrium of a zero-sum game between a Learner ($Q$-player) and an Auditor ($\lamb$-player) where $L(Q, \lamb)$ is how much the Learner must pay to the Auditor. Algorithm~\ref{msralgo}, developed by \cite{agarwal}, proceeds iteratively according to a no-regret dynamic where in each iteration, the Learner plays the best response ($\besth$) to the given play of the Auditor and the Auditor plays exponentiated gradient descent. The average play of both players over $T$ rounds are then taken as the output of the algorithm, which can be shown to converge to the saddle point $(Q^\star, \lamb^\star)$ (\cite{freund}). \cite{agarwal} shows how $\besth$ can be solved efficiently having access to the cost-sensitive classification oracle for $\Hs$ ($\text{CSC}(\Hs)$) and we have their reduction for our Equalized Odds notion of fairness written in Subroutine~\ref{besth}.
\medskip
\begin{assumption}[Cost-Sensitive Classification Oracle for $\Hs$]\label{ass:csc}
It is assumed that the proposed algorithm has access to $\text{CSC}\, (\Hs)$ which is the cost-sensitive classification oracle for $\Hs$. This oracle takes as input a set of individual-level attributes and costs $\{ X_i, C_i^0, C_i^1 \}_{i=1}^m$, and outputs $\argmin_{h \, \in \, \Hs} \sum_{i=1}^m h(X_i) C_i^1 + \left(1 - h(X_i) \right)C_i^0$. In practice, these oracles are implemented using learning heuristics.
\end{assumption}
\medskip
Note that the Learner finds $\argmin_{Q \in \Delta(\Hs)} L(Q, \lamb)$ for a given $\lamb$ of the Auditor and since the Lagrangian $L$ is linear in $Q$, the minimizer of $L(Q, \lamb)$ can be chosen to put all the probability mass on a single classifier $h \in \Hs$. Additionally, our reduction in Subroutine~\ref{besth} looks different from the one derived in Example 4 of \cite{agarwal} since we have our Equalized Odds fairness constraints formulated a bit differently from how it is formulated in \cite{agarwal}.

\begin{subroutine}
\KwIn{$\lamb$, training examples $\{ (X_i, A_i, Y_i) \}_{i=1}^m$}
\medskip
\For{$i=1, \ldots, m$}{
$C_i^0 \leftarrow \1 \{ Y_i \neq 0 \}$ \\
$C_i^1 \leftarrow \1 \{ Y_i \neq 1 \} + \frac{\lambda_{(A_i, Y_i, +)} - \lambda_{(A_i, Y_i, -)}}{\qhat_{A_i Y_i}} \1 \{ A_i \neq 0\} - \underset{\underset{a \neq 0}{a \in \A}}{\sum} \frac{\lambda_{(a, Y_i, +)} - \lambda_{(a, Y_i, -)}}{\qhat_{A_i Y_i}} \1 \{ A_i = 0\}$
}
Call $\text{CSC} \, (\Hs)$ to find $h^\star = \underset{h \in \Hs}{\argmin} \sum_{i=1}^m h(X_i) C_i^1 + \left(1 - h(X_i) \right)C_i^0$
\medskip

\KwOut{$h^\star$}
\caption{$\besth$}
\label{besth}
\end{subroutine}

\begin{algorithm}
\KwIn{fairness violation $\gamma$ \\ \ \ \ \ \ \ \ \ \ \ \ bound B, learning rate $\eta$, number of rounds $T$ \\ \ \ \ \ \ \ \ \ \ \ \ training examples $\{ (X_i, A_i, Y_i) \}_{i=1}^m$ }
\medskip
$\boldsymbol{\theta}_1 \leftarrow \boldsymbol{0} \in \R^{4(|\A|-1)}$\\
\For{$t=1, \ldots, T$}{
$\lambda_{t,k} \leftarrow B \frac{\exp(\theta_{t,k})}{1 + \sum_{k'} \exp(\theta_{t,k'})}$ for $1 \le k \le 4 (|\A|-1)$ \\
$h_t \leftarrow \besth (\lamb_t)$\\
$\boldsymbol{\theta}_{t+1} \leftarrow \boldsymbol{\theta}_{t} + \eta \, \rhat_t (h_t)$
}
$\Qhat \leftarrow \frac{1}{T} \sum_{t=1}^T h_t$, \quad $\lambhat \leftarrow \frac{1}{T} \sum_{t=1}^T \lamb_t$
\medskip

\KwOut{$(\Qhat, \lambhat)$}
\caption{exp. gradient reduction for fair classification (\cite{agarwal})}
\label{msralgo}
\end{algorithm}

\cite{agarwal} shows for any $\nu > 0$, and for appropriately chosen $\eta$ and $T$, Algorithm~\ref{msralgo} under Assumption~\ref{ass:csc} returns a pair $(\Qhat, \lambhat)$ for which
\begin{align*}
&L(\Qhat, \lambhat) \, \le \, L(Q, \lambhat) + \nu \quad \text{for all } Q \in \Delta (\Hs) \\
&L(\Qhat, \lambhat) \, \ge \, L(\Qhat, \lamb) - \nu \quad \text{for all } \lamb \in \Lambda
\end{align*}
that corresponds to a $\nu$-\textit{approximate} equilibrium of the game and it implies neither player can gain more than $\nu$ by changing their strategy (see Theorem 1 of \cite{agarwal}). They further show that any $\nu$-approximate equilibrium of the game achieves an error close to the best error one would hope to get and the amount by which it violates the fairness constraints is reasonably small (see Theorem 2 of \cite{agarwal}).

\subsection{Missing Lemmas and Proofs of Section~\ref{sec:inproc}}
\label{subsec:proofsinproc}

\begin{lemma}[Sensitivity of the Private Players to $A$]\label{sensitivity}
Let $\Delta \rhat_t$ and $\Delta \ell_t$ be the sensitivity of $\rhat_t$ (of the Auditor) and $\ell_t$ (of the Learner) respectively. We have that for all $t \in [T]$,
\begin{align*}
\Delta \rhat_t = \max_{\overset{A,A' \in \A^m}{A \sim A'}} ||\rhat_t (A) - \rhat_t(A') ||_1 \le \frac{2 |\A|}{\min_{a,y} \{ \qhat_{ay} \} \, m - 1}
\end{align*}
\begin{align*}
\Delta \ell_t = \max_{h \in \Hs} \ \max_{\overset{A,A' \in \A^m}{A \sim A'}} | \ell_t(h;A) - \ell_t(h;A') | \le \frac{2 |\A| B + 1}{\min_{a,y} \{ \qhat_{ay} \} \, m - 1}
\end{align*}
\end{lemma}
\begin{proof}[Proof of Lemma~\ref{sensitivity}] Recall that at round $t$, the private $\lamb$-player is given some $h_t \in \Hs$ and wants to calculate
\begin{align*}
\rhat_t (h_t) = \begin{bmatrix}
\fprhat_a (h_t) - \fprhat_0 (h_t) - \gamma \\ \fprhat_0 (h_t) - \fprhat_a (h_t) - \gamma \\
\tprhat_a (h_t) - \tprhat_0 (h_t) - \gamma \\ \tprhat_0 (h_t) - \tprhat_a (h_t) - \gamma
\end{bmatrix}_{\underset{a \neq 0}{a \in \A}}\in \R^{4 (|\A|-1)}
\end{align*}
privately, where for all $a \in \A$, we have that
\begin{align*}
\fprhat_a(h_t) = \frac{\qhat_{1a0}}{\qhat_{a0}} = \frac{\qhat_{1a0}}{\qhat_{1a0} + \qhat_{0a0}} \quad \quad \tprhat_a(h_t) = \frac{\qhat_{1a1}}{\qhat_{a1}} = \frac{\qhat_{1a1}}{\qhat_{1a1} + \qhat_{0a1}}
\end{align*}
Having modified one of the records in $A \in \A^m$, say $A_j = a$ is changed to $A'_j = a'$ for some $j \in [m]$, $\qhat_{\yhat_j a y_j}$ will then decrease by $1/m$ and $\qhat_{\yhat' a' y_j}$ will increase by $1/m$ where $\yhat'$ may or may not be equal to $\yhat_j$. Thus, depending on the value of $y_j$, it is then the case that
\begin{itemize}
\item if $y_j=0$: $\fprhat_a (h_t)$ and $\fprhat_{a'} (h_t)$ will change by at most $1 /\left( \min_{a,y} \{ \qhat_{ay} \} \, m - 1 \right)$.
\item if $y_j=1$: $\tprhat_a (h_t)$ and $\tprhat_{a'} (h_t)$ will change by at most $1 /\left( \min_{a,y} \{ \qhat_{ay} \} \, m - 1 \right)$.
\end{itemize}
Therefore, since each $\fprhat_a$ ($\tprhat_a$) appears twice in $\rhat_t (h_t)$ if $a\neq0$ and $2(|\A|-1)$ times if $a=0$, we have that
$$
\Delta \rhat_t \le \frac{2|\A|}{\min_{a,y} \{ \qhat_{ay} \} \, m - 1}
$$
Let's move on to the sensitivity of $\ell_t$ of the private $Q$-player. Recall that at round $t$, the $Q$-player is given some $\lamb_t \in \Lambda$ and wants to find $\argmin_{h \in \Hs} \ell_t(h) \equiv L(h,\lamb_t) = \errhat(h) + \lamb_t^\top \rhat_t (h)$ privately. It is then obvious that since $|| \lamb_t ||_1 \le B$,
\begin{align*}
\Delta \ell_t &\le \frac{1}{m} + \frac{2 |\A| B}{\min_{a,y} \{ \qhat_{ay} \} \, m - 1} \\
&\le \frac{2 |\A| B + 1}{\min_{a,y} \{ \qhat_{ay} \} \, m - 1}
\end{align*}
\end{proof}

\begin{lemma}[Accuracy of the Private Players] \label{accuracy}
At round $t$ of Algorithm~\ref{dpalgo}, let $\rhat_t = \rtilde_t - \boldsymbol{W}_t$ be the noiseless version of $\rtilde_t$ and $h_t^\star$ be the classifier given by the noiseless subroutine $\besth (\lambtilde_t)$. We have that
\begin{align*}
\text{w.p.}\ge 1-\beta/2T, \quad || \rtilde_t - \rhat_t ||_\infty \, \le \, \frac{8 |\A| \sqrt{T\ln(1/\delta)} \ln \left(8T|\A| /\beta \right)}{\left(\min_{a,y} \{ \qhat_{ay} \} \, m - 1 \right) \cdot \epsilon}
\end{align*}
\begin{align*}
\text{w.p.}\ge 1-\beta/2T, \quad L(\htilde_t, \lambtilde_t) \, \le \, L(h_t^\star, \lambtilde_t) + \frac{8 \left(2 |\A| B + 1 \right) \sqrt{T\ln \left(1/\delta \right)} \left( d_{\Hs} \ln \left(m \right) + \ln \left(2 T/\beta \right) \right)}{\left(\min_{a,y} \{ \qhat_{ay} \} \, m - 1 \right) \cdot \epsilon}
\end{align*}
\end{lemma}
\medskip
\begin{proof}[Proof of Lemma~\ref{accuracy}]
Results follow from Lemma~\ref{sensitivity}, Theorem \ref{laplacethm} and Theorem \ref{expthm} of this paper. Recall that $|\Hs(S)| \le O(m^{d_{\Hs}})$ by Sauer's Lemma.
\end{proof}

\begin{proof}[Proof of Lemma~\ref{regretq}]
This result follows directly from the accuracy of the private $Q$-player given in Lemma~\ref{accuracy}.
\end{proof}

\begin{proof}[Proof of Lemma~\ref{regretlambda}]
We follow the proof given for Theorem 1 of \cite{agarwal} and modify where necessary. Let $\Lambda' = \{ \lamb' \in \R_+^{4|\A| -3} : \ ||\lamb'||_1 = B\}$. Any $\lamb \in \Lambda$ is associated with a $\lamb' \in \Lambda'$ which is equal to $\lamb$ on the first $4 (|\A|-1)$ coordinates and has the remaining mass on the last one. Let $\rtilde'_t \in \R^{4|\A| -3}$ be equal to $\rtilde_t$ on the first $4 (|\A|-1)$ coordinates and zero in the last one. We have that for any $\lamb$ and its associated $\lamb'$, and particularly $\lambtilde_t$ and $\lambtilde_t'$ of Algorithm~\ref{dpalgo}, and all $t$
\begin{equation}\label{eq:4}
\lamb^\top \, \rtilde_t = (\lamb')^\top \, \rtilde'_t \quad \text{,} \quad \lambtilde_t^\top \, \rtilde_t = (\lambtilde_t')^\top \, \rtilde'_t
\end{equation}
Observe that with probability at least $1-\beta/2T$, $|| \rtilde_t' ||_\infty = || \rtilde_t ||_\infty \le 2 + \frac{8 |\A| \sqrt{T\ln(1/\delta)} \ln(8T|\A| /\beta)}{\left(\min_{a,y} \{ \qhat_{ay} \} \, m - 1 \right) \cdot \epsilon}$ (see Lemma~\ref{accuracy}). Thus, by Corollary 2.14 of \cite{shalev}, we have that with probability at least $1-\beta/2$, for any $\lamb' \in \Lambda'$,
\begin{align*}
\sum_{t=1}^T (\lamb')^\top \, \rtilde'_t \, \le \, \sum_{t=1}^T (\lambtilde_t')^\top \, \rtilde'_t + \frac{B \ln(4|\A|-3)}{\eta} + 4 \eta B \left( 1 + \frac{4|\A| \sqrt{T\ln(1/\delta)} \ln(8T|\A| /\beta)}{\left(\min_{a,y} \{ \qhat_{ay} \} \, m - 1 \right) \cdot \epsilon}\right)^2 T
\end{align*}
Consequently, by Equation~\ref{eq:4}, we have that with probability at least $1-\beta/2$, for any $\lamb \in \Lambda$,
\begin{equation}\label{eq:5}
\sum_{t=1}^T \lamb^\top \, \rtilde_t \, \le \, \sum_{t=1}^T \lambtilde_t^\top \, \rtilde_t + \frac{B \ln(4|\A|-3)}{\eta} + 4 \eta B \left( 1 + \frac{4|\A| \sqrt{T\ln(1/\delta)} \ln(8T|\A| /\beta)}{\left(\min_{a,y} \{ \qhat_{ay} \} \, m - 1 \right) \cdot \epsilon}\right)^2 T
\end{equation}
which completes the proof.
\end{proof}

\begin{proof}[Proof of Theorem~\ref{theorem1}]
Let
$$
R_{Q} :=\frac{8(2 |\A| B + 1) \sqrt{T\ln(1/\delta)} \left( d_{\Hs} \ln (m) + \ln (2 T/\beta) \right)}{\left(\min_{a,y} \{ \qhat_{ay} \} \, m - 1 \right) \cdot \epsilon}
$$
and
$$
R_{\lamb} := \frac{B \ln(4 |\A| -3 )}{\eta T} + 4 \eta B \left( 1 + \frac{4 |\A| \sqrt{T\ln(1/\delta)} \ln(8T|\A| /\beta)}{\left(\min_{a,y} \{ \qhat_{ay} \} \, m - 1 \right) \cdot \epsilon}\right)^2
$$
be the regret bounds of the private $Q$ and $\lamb$ players respectively, and let $\nu := R_{Q} + R_{\lamb}$. We have that for any $Q \in \Delta (\Hs(S))$, with probability at least $1 - \beta$,
\begin{align*}
L(Q, \lambtilde) &= \frac{1}{T} \sum_{t=1}^T L(Q, \lambtilde_t)   \quad \text{(by linearity of }L) \\
&\ge \frac{1}{T} \sum_{t=1}^T L(\htilde_t, \lambtilde_t) - R_{Q} \quad (\text{by Lemma~\ref{regretq}})\\
&\ge \frac{1}{T} \sum_{t=1}^T L(\htilde_t, \lambtilde) - R_{\lamb} - R_{Q} \quad (\text{by Lemma~\ref{regretlambda}}) \\
&= L(\Qtilde, \lambtilde) - \nu
\end{align*}
Now for any $\lamb \in \Lambda$, with probability at least $1 - \beta$,
\begin{align*}
L(\Qtilde, \lamb) &= \frac{1}{T} \sum_{t=1}^T L(\htilde_t, \lamb)  \quad \text{(by linearity of }L)\\
&\le \frac{1}{T} \sum_{t=1}^T L(\htilde_t, \lambtilde_t) + R_{\lamb} \quad (\text{by Lemma~\ref{regretlambda}}) \\
& \le \frac{1}{T} \sum_{t=1}^T L(\Qtilde, \lambtilde_t) + R_{\lamb} + R_{Q} \quad (\text{by Lemma~\ref{regretq}})\\
&= L(\Qtilde, \lambtilde) + \nu
\end{align*}
Therefore, with probability at least $1-\beta$,
\begin{align*}
&L(\Qtilde, \lambtilde) \, \le \, L(Q, \lambtilde) + \nu \quad \text{for all } Q \in \Delta (\Hs(S)) \\
&L(\Qtilde, \lambtilde) \, \ge \, L(\Qtilde, \lamb) - \nu \quad \text{for all } \lamb \in  \Lambda
\end{align*}
and that
\begin{align*}
\nu &= \frac{B \ln(4 |\A| -3 )}{\eta T} + 4 \eta B \left( 1 + \frac{4 |\A| \sqrt{T\ln(1/\delta)} \ln(8T|\A| /\beta)}{\left(\min_{a,y} \{ \qhat_{ay} \} \, m - 1 \right) \cdot \epsilon}\right)^2 \\
& \ \ \ + \frac{8(2 |\A| B + 1) \sqrt{T\ln(1/\delta)} \left( d_{\Hs} \ln (m) + \ln (2 T/\beta) \right)}{\left(\min_{a,y} \{ \qhat_{ay} \} \, m - 1 \right) \cdot \epsilon}
\end{align*}
Plugging in the proposed values of $T$ and $\eta$ in Algorithm~\ref{dpalgo} results in
\begin{align*}
\nu = \widetilde{O} \left( \frac{B}{\min_{a,y} \{ \qhat_{ay} \}} \sqrt{\frac{|\A| \sqrt{ \ln(1/\delta)} \left( d_{\Hs} \ln (m) + \ln (1/\beta) \right)}{m \, \epsilon}} \right)
\end{align*}
where we hide further logarithmic dependence on $m$, $\epsilon$, and $|\A|$ under the $\widetilde{O}$ notation.
\end{proof}
\medskip
The following two lemmas are taken from \cite{agarwal} and are used in the proof of Theorem \ref{theorem2} and Theorem \ref{theorem4}.
\medskip
\begin{lemma}[Empirical Error Bound \cite{agarwal}]\label{lemma2msr}
Let $(\Qtilde, \lambtilde)$ be any $\nu$-approximate solution of the game described in section \ref{sec:inproc},i.e.,
\begin{align*}
&L(\Qtilde, \lambtilde) \, \le \, L(Q, \lambtilde) + \nu \quad \text{for all } Q \in \Delta (\Hs) \\
&L(\Qtilde, \lambtilde) \, \ge \, L(\Qtilde, \lamb) - \nu \quad \text{for all } \lamb \in  \Lambda
\end{align*}
For any $Q$ satisfying the fairness constraints of the fair ERM problem, we have that
$$
\errhat (\Qtilde) \le \errhat \left(Q \right) + 2 \nu
$$
\end{lemma}

\medskip

\begin{lemma}[Empirical Fairness Violation \cite{agarwal}]\label{lemma3msr}
Let $(\Qtilde, \lambtilde)$ be any $\nu$-approximate solution of the game described in section \ref{sec:inproc}, i.e., 
\begin{align*}
&L(\Qtilde, \lambtilde) \, \le \, L(Q, \lambtilde) + \nu \quad \text{for all } Q \in \Delta (\Hs) \\
&L(\Qtilde, \lambtilde) \, \ge \, L(\Qtilde, \lamb) - \nu \quad \text{for all } \lamb \in  \Lambda
\end{align*}
and suppose the fairness constraints of the fair ERM problem are feasible. Then the distribution $\Qtilde$ satisfies
$$
\max_{a \, \in \, \A} \left \vert \fprhat_a (\Qtilde) - \fprhat_0 \left(\Qtilde \right) \right \vert \le \gamma + \frac{1 + 2 \nu}{B}
$$
$$
\max_{a \, \in \, \A} \left \vert \tprhat_a (\Qtilde) - \tprhat_0 \left(\Qtilde \right) \right \vert \le \gamma + \frac{1 + 2 \nu}{B}
$$
\end{lemma}

\begin{proof}[Proof of Theorem \ref{theorem2}]
The results follow from Theorem~\ref{theorem1}, Lemma \ref{lemma2msr}, and Lemma \ref{lemma3msr}.
\end{proof}
\medskip
\begin{proof}[Proof of Theorem~\ref{theorem4}]
The stated bound on $\errhat (\Qtilde)$ follows from Lemma \ref{lemma2msr}. Let's now prove the bound on fairness violation. Let, for all $a\in \A$, $\beta_a  := (\fprhat_0 (\Qtilde)  - \fprhat_a(\Qtilde) -\gamma)_+$ and $\bar{\beta}_a  := (\fprhat_a (\Qtilde)  - \fprhat_0(\Qtilde) -\gamma)_+$. Notice at most one of $\beta_a$ and $\bar{\beta}_a$ can be positive.

We are going to construct some deviating strategies: $Q$ and $\lamb$. As shown in the previous  subsection, we know $(\Qtilde,\lambtilde)$ is a $\nu$-approximate equilibrium of the zero-sum game. It implies
\[
L(\Qtilde, \lamb) - \nu  \leq  L(\Qtilde,\lambtilde)  \leq L(Q,\lambtilde) + \nu.
\]

Define $Q = \frac{1}{1 +\sum_{a \in \A} (\beta_a + \bar{\beta}_a)}(\Qtilde   + \sum_a \beta_a h_a + \hat{\beta}_a \hat{h}_a)$. It is easy to see that, for all $a\in \A$,
\[
\Delta \fprhat_a(Q) \preceq \gamma.
\]
Then we have
\begin{align*}
&L(Q,\lambtilde) + \nu\\
\leq &\errhat(Q) + \nu\\
\leq &\errhat\left( \frac{1}{1 +\sum_{a \in \A} (\beta_a + \bar{\beta}_a)}(\Qtilde   + \sum_a \beta_a h_a + \hat{\beta}_a \hat{h}_a)\right) + \nu\\
\leq &\frac{1}{1 +\sum_{a \in \A} (\beta_a + \bar{\beta}_a)} \errhat(\Qtilde) +\frac{\sum_{a \in \A} (\beta_a + \bar{\beta}_a)}{1 +\sum_{a \in \A} (\beta_a + \bar{\beta}_a)} + \nu\\
\leq &  \errhat(\Qtilde) + \sum_{a \in \A} (\beta_a + \bar{\beta}_a)+ \nu \\
\leq &  \errhat(\Qtilde) + (|\A|-1)\cdot (\max_{a\in \A} | \fprhat_a (\Qtilde)  - \fprhat_0(\Qtilde)| - \gamma)_+ + \nu. 
\end{align*}

Define $\lamb$ to have $B$ in the coordinate which corresponds to $\arg\max_{a\in \A} | \fprhat_a (\Qtilde)  - \fprhat_0(\Qtilde)|$ and 0 in other coordinates. Then we have
\[
L(\Qtilde, \lamb) -\nu = \errhat(\Qtilde) + B(\max_{a\in \A} | \fprhat_a (\Qtilde)  - \fprhat_0(\Qtilde)| - \gamma) -\nu 
\]
To sum up, we get
\[
\errhat(\Qtilde) + B(\max_{a\in \A} | \fprhat_a (\Qtilde)  - \fprhat_0(\Qtilde)| - \gamma) -\nu  \leq  \errhat(\Qtilde) + (|\A|-1)\cdot (\max_{a\in \A} | \fprhat_a (\Qtilde)  - \fprhat_0(\Qtilde)| - \gamma)_+ + \nu. 
\]

This implies
\[
\max_{a \, \in \, \A} | \fprhat_a (\Qtilde)  - \fprhat_0(\Qtilde)| \leq \gamma+  \frac{2\nu}{B - (|\A|-1)}.
\]
which completes the proof.
\end{proof}
\medskip
\begin{proof}[Proof of Theorem \ref{thm:separation}]
First consider the case where $\Hs = \{h_0,h_U\}$. Choose data set $D$ of size $m$ as follows: $m/2$ individuals with $(A=R,X=V,Y=0)$; $m/4$ individuals with $(A=B,X=U,Y=1)$, $m(1-\gamma)/4$ individuals with $(A=B,X=V,Y=0)$ and $m\gamma/4$ individuals with $(A=B,X=U,Y=0)$. For this data set, it is easy to check that $h_U$ has error $\gamma /4$ and $h_U$ satisfies the fairness constraint. So $f(D) \leq \gamma /4$. Now consider $D$'s neighboring data set $D'$ by changing one individual with $(A=B,X=V,Y=0)$ to $(A=B,X=U,Y=0)$. For $D'$, the classifier which satisfies the fairness constraint and has the minimum error rate is $\frac{1}{4+\gamma m} (4 h_0 + \gamma m h_U)$. Therefore
\[
f(D') = \frac{1}{4+\gamma m}\left(4 \cdot \frac{1}{4} + \gamma m \cdot \frac{m\gamma /4 + 1}{m}  \right) = \frac{\gamma}{4} + \frac{1}{4+\gamma m} .
\]
implying that $|f(D) - f(D')| = \Omega(1/(\gamma m))$ and the sensitivity of $f$ is $\Omega(1/(\gamma m))$.

Now consider the case where $\Hs = \{h_0, h_U, h_R,h_B\}$. It suffices to show that $f(D') \leq f(D) + O(1/m)$ for any neighboring data sets $D$ and $D'$. Let $Q^*$ be the classifier with minimum error rate on data set $D$. We have $f(D) = \errhat(Q^*,D)$ and we know $|\fprhat_R(Q^*,D)  - \fprhat_B(Q^*,D) |\leq \gamma$ (we put $D$ into the arguments of $\errhat$ and $\fprhat$ as we are talking about two different data sets).  For data set $D'$, there are two cases.
\begin{itemize}
\item The case when $|\fprhat_R(Q^*,D')  - \fprhat_B(Q^*,D') |\leq \gamma$: In this case, we have
\[
f(D') \leq \errhat(Q^*,D') \leq \errhat(Q^*,D) + 1/m = f(D) + 1/m.
\]
\item The case when $|\fprhat_R(Q^*,D')  - \fprhat_B(Q^*,D') |> \gamma$: Wlog let's assume $\fprhat_R(Q^*,D')  - \fprhat_B(Q^*,D') > \gamma$. And let $\alpha = \fprhat_R(Q^*,D')  - \fprhat_B(Q^*,D') - \gamma$. We know $\alpha > 0$ and we also have
\[
\alpha = \fprhat_R(Q^*,D')  - \fprhat_B(Q^*,D') - \gamma \leq \fprhat_R(Q^*,D)  - \fprhat_B(Q^*,D) - \gamma + 2/(Cm) \leq 2/(Cm).
\]
Now define $Q' = \frac{1}{1+\gamma + \alpha}\left( (1+\gamma) Q^* + \alpha h_B\right)$. We have
\[
\fprhat_R(Q',D')  - \fprhat_B(Q',D') =\frac{1}{1+\gamma + \alpha}\left((1+\gamma) (\fprhat_R(Q^*,D')  - \fprhat_B(Q^*,D')) - \alpha\right) = \gamma.
\]
Therefore
\begin{align*}
f(D') &\leq \errhat(Q',D') \\
&\leq \frac{1}{1+\gamma + \alpha}\left((1+\gamma) \, \errhat(Q^*,D')  + \alpha \, \errhat(h_B,D') \right) \\
&\leq f(D) + 1/m + \alpha \\
&\leq f(D) + O(1/m)
\end{align*}
\end{itemize}
\end{proof} 


\end{document}